%% file: main.tex
\documentclass[a4paper]{article}
\usepackage{amsmath}
\usepackage{amssymb}
\usepackage{mathtools}
\usepackage{authblk}
\usepackage{hyperref,url,dsfont}
\usepackage{bookmark}
\usepackage{natbib}
\usepackage{fullpage}

\input{preamble}

\title{Non-Parametric Rehearsal Learning via Conditional Mean Embeddings}

\author[1,2]{Wen-Bo Du}
\author[1,2]{Tian-Zuo Wang}
\author[1,2]{Han-Jia Ye}
\author[1,2]{Zhi-Hua Zhou}
\affil[1]{National Key Laboratory for Novel Software Technology, Nanjing University, China}
\affil[2]{School of Artificial Intelligence, Nanjing University, China}
\date{}

\begin{document}

\maketitle

\input{sections/abstract}
\input{sections/introduction}
\input{sections/setup_identification}
\input{sections/method_objective}

\input{sections/estimator_optimization}
\input{sections/theory}
\input{sections/experiments}
\input{sections/limitations_conclusion}

\bibliographystyle{plainnat}
\bibliography{references}

\appendix
\input{appendix/appendix_main}

\end{document}

%% file: preamble.tex
\usepackage{microtype}
\usepackage{xcolor}
\usepackage{graphicx}
\usepackage{makecell}
\usepackage{subcaption}
\usepackage{booktabs}
\usepackage{wrapfig}
\usepackage{enumitem}
\usepackage{amsmath}
\usepackage{amssymb}
\usepackage{mathtools}
\usepackage{amsthm}
\usepackage{bm}
\usepackage{algorithm}
\usepackage{algorithmic}

\usepackage{hyperref}
\usepackage{url}
\usepackage[capitalize,noabbrev]{cleveref}

\hypersetup{
    colorlinks=true,
    citecolor=blue!80!black,
    linkcolor=red!80!black,
    urlcolor=blue!80!black,
    filecolor=blue!80!black
}

\theoremstyle{plain}
\newtheorem{theorem}{Theorem}
\newtheorem{proposition}{Proposition}
\newtheorem{lemma}{Lemma}

\newtheorem*{prop1}{Proposition 1}
\newtheorem*{theorem1}{Theorem 1}
\newtheorem*{theorem2}{Theorem 2}

\theoremstyle{definition}
\newtheorem{definition}{Definition}
\newtheorem{assumption}{Assumption}

\theoremstyle{remark}

\newif\ifshowcomments
\showcommentsfalse

\makeatletter
\newcommand{\safe@def}[2]{%
    \@ifundefined{#1}{\expandafter\def\csname #1\endcsname{\ensuremath{\mathbf{#1}}}}{}%
    \@ifundefined{#2}{\expandafter\def\csname #2\endcsname{\ensuremath{\mathbf{#2}}}}{}%
    \@ifundefined{s#2}{\expandafter\def\csname s#2\endcsname{\ensuremath{\mathcal{#2}}}}{}%
    \@ifundefined{bb#2}{\expandafter\def\csname bb#2\endcsname{\ensuremath{\mathbb{#2}}}}{}%
}

\safe@def{a}{A} \safe@def{b}{B} \safe@def{c}{C} \safe@def{d}{D}
\safe@def{e}{E} \safe@def{f}{F} \safe@def{g}{G} \safe@def{h}{H}
\safe@def{i}{I} \safe@def{j}{J} \safe@def{k}{K} \safe@def{l}{L}
\safe@def{m}{M} \safe@def{n}{N} \safe@def{o}{O} \safe@def{p}{P}
\safe@def{q}{Q} \safe@def{r}{R} \safe@def{s}{S} \safe@def{t}{T}
\safe@def{u}{U} \safe@def{v}{V} \safe@def{w}{W} \safe@def{x}{X}
\safe@def{y}{Y} \safe@def{z}{Z}
\makeatother

\renewcommand{\a}{\ensuremath{\mathbf{a}}}
\renewcommand{\u}{\ensuremath{\mathbf{u}}}
\renewcommand{\h}{\ensuremath{\mathbf{h}}}
\renewcommand{\T}{\mathsf{T}}

\newcommand{\gray}{\color{gray}}

\def \eg {\textit{e.g.}}
\def \ie {\textit{i.e.}}

\def \wrt {w.r.t.}
\def \Fig {Fig.}
\def \Eq {Eq.}
\def \Sec {Sec.}
\def \PA {\mathbf{PA}}
\newcommand{\aleq}{\mathop{=}\limits^{a}}

%% file: sections/abstract.tex
\begin{abstract}
In machine learning, a critical class of decision-related problems concerns preventing predicted undesirable outcomes, referred to as the \textit{avoiding undesired future} (AUF) problem.
To address this, the \textit{rehearsal learning} framework has been proposed to model influence relations for effective decisions.
However, existing rehearsal methods rely on restrictive parametric assumptions such as linear systems or additive noise, limiting their practical applicability. 
In this paper, we propose the first non-parametric rehearsal learning approach for AUF without assuming specific functional forms of data generation processes.
Specifically, we use kernel machinery to reformulate the AUF objective into a unified representation that disentangles desirability modeling from action-induced distributional changes. 
To handle the discontinuity of desirability indicator, we present a smooth Probit surrogate and provide an approximation error bound. Meanwhile, we capture the action-induced changes via conditional mean embeddings, and develop a kernel ridge regression based nested estimator for AUF objective with consistency guarantees.
Such a formulation naturally accommodates nonlinear systems and non-additive noise, and empirical results on synthetic and real-data-derived semi-synthetic benchmarks demonstrate the effectiveness and flexibility of our approach.
\end{abstract}

%% file: sections/introduction.tex
\section{Introduction}
  \label{sec:intro}
Machine learning (ML) have achieved great success in various real-world prediction tasks~\citep{lecun2015deep}. 
Instead of solely focusing on prediction, \citet*{zhouzh2022rehearsal} emphasizes another important issue, \ie, if an ML model predicts that something undesired is going to occur, how to find effective actions to prevent it from happening. 
This problem is known as \emph{avoiding undesired future} (AUF) problem.
Consider a motivating example of bank lending illustrated in Fig.~\ref{fig:auf_example}, where the bank seeks to proactively optimize its loan approval strategy.
The objective is to maximize the probability that the target outcomes $\mathbf{Y}$ (\eg, repayment rate and return on investment (ROI)) satisfy specific expectations, formalized as falling into a desired region $\mathcal{S}$ (\eg, $\{ \mathbf{y}\, |\,\text{repayment rate}\geq \alpha, \text{ROI} \geq \beta \}$).
In practice, the bank typically observes user profiles $\mathbf{X}$ (\eg, credit score and debt-to-income ratio) to predict \(\Y\).
If the model anticipates a high likelihood of an undesired future (\ie, $\mathbf{Y} \notin \mathcal{S}$), it becomes imperative to actively alter actionable intermediate variables $\mathbf{Z}_A$ (\eg, interest rate) to steer \(\Y\) towards $\mathcal{S}$.

A main challenge in such scenario lies in the fact that historical records are collected observationally, where the values of $\mathbf{Z}_A$ could be confounded by unactionable factors (\eg, market fluctuations) that affect both $\mathbf{Z}_A$ and $\mathbf{Y}$.
Consequently, the data distribution under such active alteration differs significantly from the conditional distribution on the historical data.
Furthermore, in high-stakes scenarios such as financial lending and healthcare, actively exploring alterations via trial-and-error is prohibitive due to the risk of financial loss in lending and ethical concerns in medical decision-making.
Thus, finding optimal decision alterations to maximize the probability of \(\Y\in\sS\) (\ie, the AUF probability) solely based on historical data becomes the core objective of the AUF problem.

\begin{figure*}
  \centering
      \includegraphics[width=0.98\linewidth]{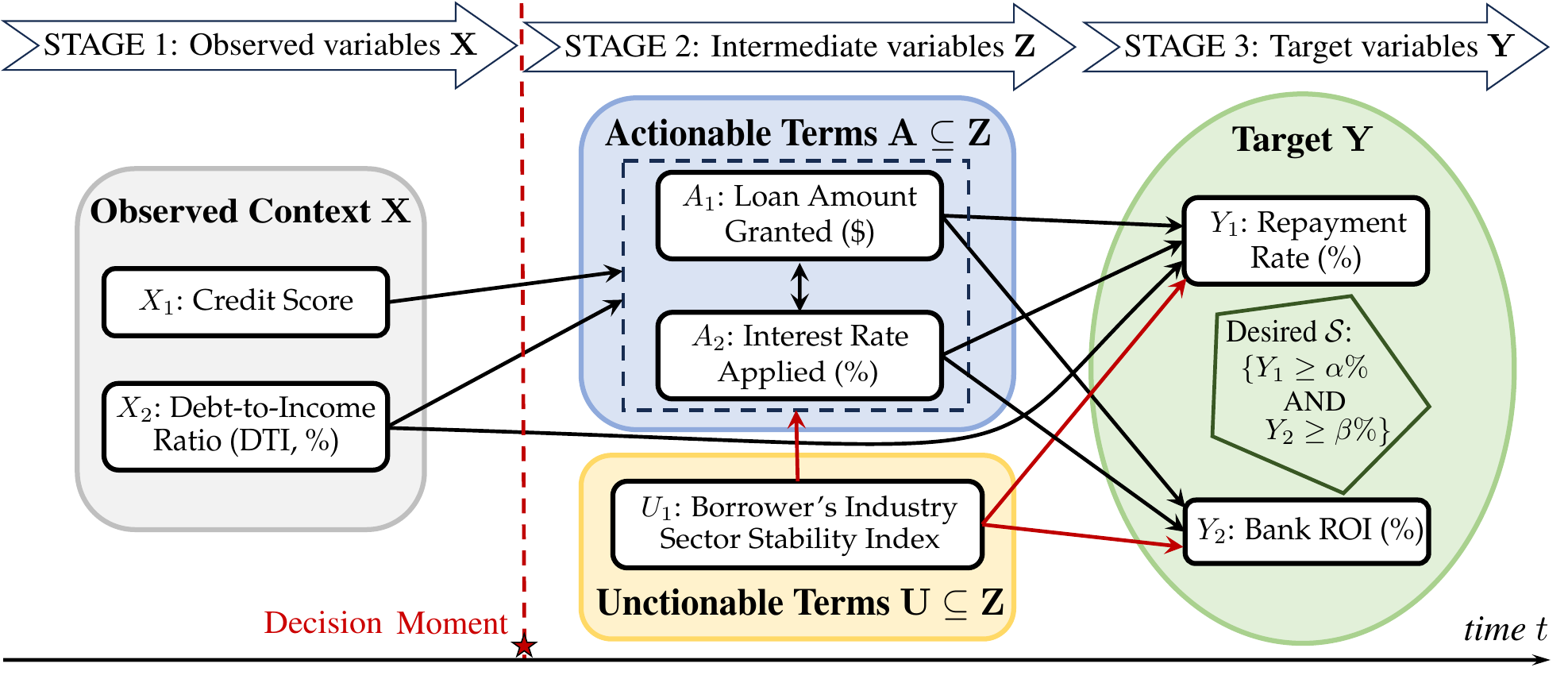}
      \caption{An example illustrating the bank lending AUF scenario.} % 左图的标题
  \label{fig:auf_example} % 总图的引用标签
  \vskip -0.1in
\end{figure*}

To address this, one solution is by \emph{rehearsal learning} framework~\citep{zhouzh2022rehearsal}, which leverages fine-grained structural information, including both direct and bi-directed influence, to guide the selection of decision alterations. While existing rehearsal-based methods have made progress, particularly under assumptions of linear additive noise models~\citep{qint2023srm,du2025CARE}, such assumptions rarely hold in complex real-world systems, posing practical challenges. 
Fundamentally, these methods are parametric (and predominantly assumes linear Gaussian) that heavily relies on the alignment between the true data generation mechanism and the hypothesized parametric forms. Once such alignment is not achievable, the resulting decision strategy suffers from severe performance degradation due to model misspecification.
While \citet{qin2025gradient} attempts to accommodate non-linear scenarios, it remains parametric, and resorts to a heuristic objective which can deviate substantially from the true AUF probability.
Consequently, this approach could not guarantee that the suggested decision aligns with the true AUF goal, often yielding suboptimal empirical results.
% Consequently, this approach lacks the theoretical guarantee that its optimization target aligns with the actual goal of AUF, often yielding suboptimal results.

Considering the limitations, in this paper, we propose the first non-parametric rehearsal learning framework that addresses AUF without relying on specific functional form assumptions.
Our key idea is to use kernel machinery to make the AUF objective both \emph{estimable} from observational data and \emph{optimizable} over potential actions.
Specifically, we leverage conditional mean embeddings (CMEs) to reformulate the original AUF probability maximization as maximizing an inner product between a desirability function and an action-dependent embedding of target variables $\mathbf{Y}$ within a reproducing kernel hilbert space (RKHS). 
To handle the discontinuity of the desirability indicator $\mathbb{I}(\mathbf{Y}\in\mathcal{S})$ that does not reside in standard continuous kernel spaces, we introduce a smooth surrogate $w_\eta$ with an approximation error bound to the indicator.
Although surrogate smoothing and action-dependent evaluation are separate challenges, they can be \emph{unified} under the RKHS representation when the action-dependent component is captured via CMEs of identifiable conditional distributions from observational data.
Building on this formulation, we establish identifiability of the smoothed AUF objective and develop a kernel ridge regression (KRR) based nested estimator using historical observational samples with established statistical consistency.
Finally, we propose a practical multi-start projected gradient ascent procedure to optimize the estimator and obtain AUF decisions.
For empirical evaluation, we further extract a real-data-derived semi-synthetic benchmark from NHANES 2011--2018~\citep{nhanes_cdc}, a diabetes-related statistical dataset, on which our method improves AUF probability from $0.402$ under no-action baseline to $0.596$. 
% In Tab.~\ref{tab:main_results}, the largest gap between our method and the no-action baseline is $77.6$ percentage points on Lin-Syn1 ($0.942$ vs.\ $0.166$), while the largest gap among the nonlinear benchmarks is $39.8$ percentage points on Non-Syn2 ($0.584$ vs.\ $0.186$).
Our contributions can be summarized as:
\begin{enumerate}[leftmargin=12pt, itemsep=-0.0in]
    \item We propose the first non-parametric rehearsal learning approach for AUF problem.
    By representing action-dependent target distributions with conditional mean embeddings, the proposed method avoids restrictive parametric model assumptions and is better suited to complex nonlinear systems.
    \item We present a Probit-smoothed KRR-based nested estimator for the AUF objective and establish identifiability of the smoothed objective. We further establish finite-sample error bounds and consistency guarantees, with a projected-gradient algorithm for optimizing candidate actions.
    \item We empirically evaluate the proposed method across both synthetic and real-world-derived settings.
    Besides linear/nonlinear synthetic benchmarks and the standard Bermuda benchmark, we establish and introduce a semi-synthetic NHANES diabetes benchmark for the AUF problem, where our method shows strong performance against various baselines.
\end{enumerate}

%% file: sections/setup_identification.tex
\section{Background}  
  In this section, we briefly review several key concepts.
%    and definitions that are essential for our approach.
  \subsection{Influence structure and feasible alterations}
  \label{sec:order_alter}
  We assume variables generating in a specific partial order $\mathcal{O}$ that records which variables can occur before others in the data-generating process: if $V_i \prec_{\mathcal{O}} V_j$, then $V_i$ precedes $V_j$ in every topological ordering compatible with $\mathcal{O}$.
  In this paper, the order is treated as given side information, which can be learned from observational data~\citep{tao2026order}.
  The variables are partitioned into context variables $\X$, intermediate variables $\Z$, and outcome variables $\Y$, with block-level ordering \(\X \prec_{\mathcal{O}} \Z \prec_{\mathcal{O}} \Y\).

  The context $\X$ is observed before decision-making, while $\Y$ is observed after the decision and is desired to fall in a target region $\sS$.
  The intermediate variables $\Z$ may contain both variables that can be acted upon and variables that cannot.
  % More specifically, $\Z$ will be decomposed into pre-alteration non-actionable variables $\U$, actionable variables $\A$, and post-alteration non-actionable variables $\U^{\text{post}}$ in Def.~\ref{def:Z_decompose}.
  The underlying data-generating process may be compatible with richer structured equation models, but our main results require only the known order.

  In structural causal models~\citep{spirtes2000causation}, $do(V_i=v_i)$ denotes a structural operation that cuts the dependence of $V_i$ on its original parents and replaces its structural assignment with the constant value $v_i$.
  For the AUF problem, we further need to distinguish which variables can actually be acted upon in a given real-world decision problem.
  We call a variable $V_i$ actionable if its feasible alteration set $\Delta_{V_i}$ is non-empty and the decision-maker can take an action to set $V_i$ to some value $v_i \in \Delta_{V_i}$.
  For such an actionable variable, we write $V_i \aleq v_i$ to indicate a feasible alteration.
  For a vector of actionable variables $\A$, we write $\A \aleq \a$, or equivalently use the shorthand $\mathring{\a}$, to make explicit that the operation is applied to actionable nodes.
  The notation $P(\cdot \mid \x, \mathring{\a})$ denotes the post-alteration distribution after observing context $\x$ and applying this feasible alteration.

  \subsection{Conditional mean embeddings}
  Let $U$ and $V$ be random variables in measurable domains $\sU$ and $\sV$, endowed with kernels $k_u: \sU \times \sU \to \bbR$ and $k_v: \sV \times \sV \to \bbR$, along with their associated RKHSs, denoted by $\sH_{k_u}$ and $\sH_{k_v}$.
  We assume the kernels are continuous, bounded, and positive-definite. 
  The conditional mean embedding (CME)~\citep{song2009hilbert, song2013kernel} maps the conditional distribution $P(V | u)$ to an element $\mu_{V|u} \in \sH_{k_v}$. 
  Formally, it is defined as the expected feature map of $V$ conditioned on $u$:
  \begin{equation*}
      \mu_{V|u} := \bbE_{V \sim P(\cdot \mid u)} [k_v(v, \cdot)] = \int_{\sV} k_v(v, \cdot) \, \mathrm{d}P(v \mid u).
  \end{equation*}
  A crucial property of CME is that it enables the evaluation of conditional expectations for any function $g \in \sH_{k_v}$ via the inner product in $\sH_{k_v}$ that \(\bbE[g(V) \mid u] = \langle g, \mu_{V|u} \rangle_{\sH_{k_v}}\) (reproducing property), bridging the integral operation of expectation with the algebraic operation in the Hilbert space.
  
  % Empirically, given a dataset $\sD = \{(u_i, v_i)\}_{i=1}^N$, the CME is approximated by a weighted sum of observed feature maps:
  % \begin{equation}
  %     \hat{\mu}_{V|u} := \sum_{i=1}^N \beta_i(u) \phi_v(v_i).
  % \end{equation}
  % The weights $\beta_i(u)$ capture the relevance of the $i$-th sample to the query $u$ and are derived via regularized inversion in $\sH_{k_u}$ (details in Section~\ref{sec:your_krr_section}).

%% file: sections/method_objective.tex
\section{Our approach}
\label{sec:method}

In this section, we present our proposed approach. 
We first formulate the AUF objective and introduce a smooth surrogate to handle its discontinuous. 
Subsequently, we propose a kernel ridge regression (KRR) based nested estimator to learn this objective from observational training data and develop a gradient-based optimization algorithm.
Finally, we provide theoretical analysis establishing approximation guarantees of the smooth surrogate and the convergence rate of the proposed estimator.

\subsection{Formulation of the AUF problem}
\label{sec:problem_formulation}

Given historical data $\mathcal{D} = \{(\mathbf{x}_i, \mathbf{z}_i, \mathbf{y}_i)\}_{i=1}^N$ (collected from observational distribution, distinct from the alterational one), we focus on the decision moment where the context $\mathbf{x}$ has just been observed, while the subsequent intermediate variables $\mathbf{Z}$ and outcomes $\mathbf{Y}$ are yet to materialize.
Our objective is to determine the optimal values for the actionable subset $\mathbf{A} \subseteq \mathbf{Z}$ that maximize the probability of $\mathbf{Y}$ falling into a target desired region $\mathcal{S}$ under the feasible alteration $\mathring{\mathbf{a}}$:
\begin{equation}
\label{eq:auf_opt}
\begin{aligned}
    \arg\max_{\mathbf{a}} \quad & \mathbb{P}(\mathbf{Y} \in \mathcal{S} \mid \X=\mathbf{x}, \A\aleq\mathbf{a}) \\
    \text{s.t.} \quad & \mathbf{a}_{\text{left}}\, \preccurlyeq\, \mathbf{a}\, \preccurlyeq\, \mathbf{a}_{\text{right}},
\end{aligned}
\end{equation}
where $\preccurlyeq$ denotes element-wise inequality, and $\mathbf{a}_{\text{left}}, \mathbf{a}_{\text{right}}$ are constant vectors specifying the feasible alteration range.
Notably, unlike prior works that rely on parametric structural equations governed by $\bm{\theta}$, our formulation is non-parametric and does not assume specific functional forms for the generating process.
The desired region $\mathcal{S}$ is specified as a convex polytope defined by $l$ linear constraints:
\begin{equation}
    \label{eq:region}
    \sS = \left\{ \y \in \bbR^{d_y} \mid \m_k^\T\y \leq b_k, \, k=1, \dots, l \right\}.
\end{equation}

\subsection{Smoothed desirability function with identifiability}
\label{sec:smooth_identifiability}

To facilitate the optimization of Eq.~\eqref{eq:auf_opt}, we first reformulate the AUF probability as the expectation of a desirability indicator variable. Let $\mathbb{I}(\mathbf{y} \in \mathcal{S})$ be the indicator function which takes the value 1 if $\mathbf{y} \in \mathcal{S}$ and 0 otherwise. The objective function is equivalently expressed as:
\begin{equation*}
    \mathbb{P}(\mathbf{Y} \in \mathcal{S} \mid \mathbf{x}, \mathring{\mathbf{a}}) = \mathbb{E}_{\mathbf{y} \sim P(\cdot \mid \mathbf{x}, \mathring{\mathbf{a}})} [\mathbb{I}(\mathbf{y} \in \mathcal{S})].
\end{equation*}
Directly maximizing this expectation presents two fundamental challenges regarding optimization and estimation:
(i) The indicator function $\mathbb{I}(\cdot)$ is discontinuous and has zero gradients almost everywhere. 
From a parametric perspective, calculating gradients \wrt $\mathbf{a}$ is intractable due to the lack of gradient signals from the objective itself. From a non-parametric kernel perspective, strictly discontinuous functions do not reside in the RKHS associated with common continuous kernels (\eg, Gaussian RBF), rendering standard kernel methods inapplicable for direct estimation.
% This distribution generally differs from the observational conditional $P(\Y \mid \mathbf{x}, \mathbf{a})$, thus cannot be directly estimated from the historical data $\mathcal{D}$ due to the presence of the variables influencing both $\A$ and $\Y$.
(ii) The objective depends on the alterational distribution 
$P(\Y \mid \mathbf{x}, \mathring{\mathbf{a}})$, which describes the outcome distribution after applying a feasible alteration. 
This distribution generally differs from the observational conditional 
$P(\Y \mid \mathbf{x}, \mathbf{a})$ and cannot be directly estimated from historical data $\mathcal{D}$ due to the presence of the variables influencing both $\A$ and $\Y$.

To address the first challenge, we introduce a smooth surrogate to approximate the hard indicator. Leveraging the definition of $\mathcal{S}$ in Eq.~\eqref{eq:region}, we define the smoothed desirability function $w_\eta(\mathbf{y})$ as:
\begin{equation}
    \label{eq:probit}
    w_\eta(\mathbf{y}) \triangleq \prod_{k=1}^{l} \Phi \left(\eta(b_k - \mathbf{m}_k^\top \mathbf{y})\right),
\end{equation}
where $\Phi(\cdot)$ is the cumulative distribution function (CDF) of the standard normal distribution, and $\eta > 0$ is a scaling parameter controlling the sharpness of the approximation. 
% and \(\{b_k, \m_k\}\)s are defined in Eq.~\eqref{eq:region}.

\setlength{\intextsep}{0pt}
\begin{wrapfigure}[15]{r}{0.58\textwidth}
    \centering
    \includegraphics[width=0.97\linewidth]{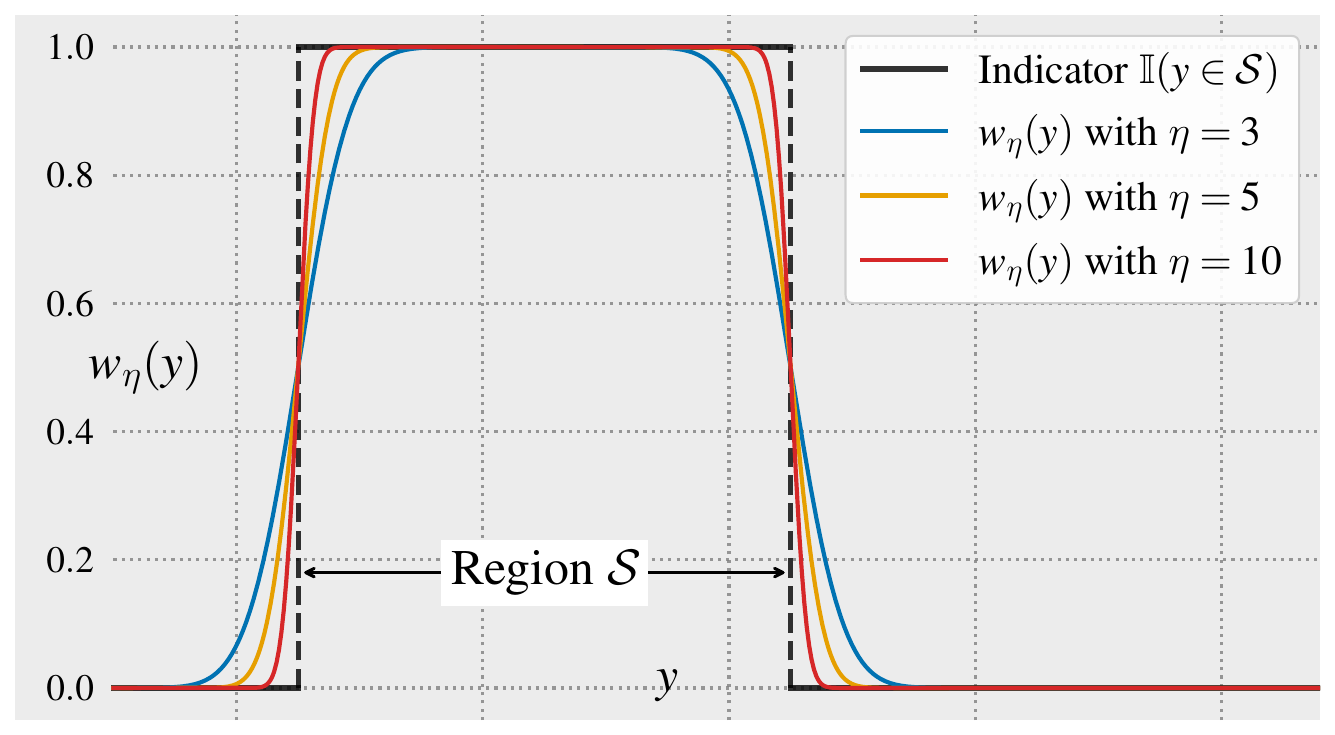}
    \caption{An example of smooth surrogate $w_\eta(\cdot)$ approximating hard indicator $\mathbb{I}(\cdot \in \mathcal{S})$ with varying scaling $\eta$.}
    \label{fig:smooth_surrogate_comparison}
\end{wrapfigure}

As illustrated in Fig.~\ref{fig:smooth_surrogate_comparison}, $w_\eta(\mathbf{y})$ converges to $\mathbb{I}(\mathbf{y} \in \mathcal{S})$ as $\eta$ increases. 
While the visualization depicts a univariate case, this approximation applies equally to the high-dimensional convex polytopes defined in Eq.~\eqref{eq:region}.
Beyond merely enabling differentiability, this surrogate offers distinct advantages.
First, $w_\eta(\cdot)$ replaces the discontinuous constraints with a function that is $C^\infty$-continuous, ensuring numerical stability for gradient-based solvers. 
Second, unlike the indicator function which yields zero gradients for any $\mathbf{y} \notin \mathcal{S}$, $w_\eta(\mathbf{y})$ provides non-zero gradients even in the infeasible region. 
This property is particularly valuable when positive samples ($\mathbf{y} \in \mathcal{S}$) are scarce, in which case the surrogate provides an informative signal distinguishing a ``near-miss'' from a significant failure, thereby effectively guiding the optimization toward the desired region.

To address the second challenge, we establish the identifiability of the AUF objective under the smooth surrogate. 
We first formalize the structure of the intermediate variables \(\Z\).

\begin{definition}[Decomposition of $\Z$]
    \label{def:Z_decompose}
    Given the known partial order $\mathcal{O}$, we use a sufficient disjoint partition $\Z = \U \cup \A \cup \U^{\text{post}}$ of the intermediate variables, where (i) the subset $\A \subseteq \Z$ contains variables that can be directly altered by the decision-maker; (ii) the subset $\U$ contains the non-actionable pre-alteration variables in this sufficient decomposition; and (iii) the subset $\U^{\text{post}} := \Z \setminus (\U \cup \A)$ contains the remaining non-actionable intermediate variables that occur after the alteration; these variables may lie downstream of $\A$ and are not adjusted for.
\end{definition}

% Based on this decomposition, the interventional expectation can be expressed using only observational conditional distributions.

\begin{proposition}[Identifiability of the AUF Objective]
    \label{prop:identifiable}
    Let $\Z = \U \cup \A \cup \U^{\text{post}}$ be the decomposition defined in Def.~\ref{def:Z_decompose}. The AUF objective defined by the smooth desirability approximation $w_\eta(\cdot)$ under the feasible alteration $\mathring{\a}$ is identifiable from the observational distribution as:
    \begin{equation*}
      \bbE_{\y\sim P\left(\cdot \mid \x, \mathring{\a}\right)}\!\!\left[w_\eta(\Y)\right]
      =
      \bbE_{\u\sim P(\cdot \mid \x)}\!\!\left[\bbE_{\y\sim P(\cdot\mid \x, \u, \a)}\!\!\left[ w_\eta(\Y) \right]\right].
    \end{equation*}
\end{proposition}
\vskip -0.05in

Prop.~\ref{prop:identifiable} transforms the alterational objective into a nested expectation over observational distributions. 
To streamline the notation for the subsequent estimation, we formally define the smoothed AUF objective as \(J_\eta(\a; \x) \triangleq \bbE_{\y\sim P(\cdot \mid \x, \mathring{\a})}[w_\eta(\Y)]\) and let $\h \triangleq (\mathbf{x}, \mathbf{u}, \mathbf{a})$ denote the concatenated vector comprising the observed context, pre-alteration covariates, and the action variables. 
Accordingly, we define a kernel function $k_h(\cdot, \cdot)$ over the domain of $\h$.

%% file: sections/estimator_optimization.tex
\subsection{Surrogate AUF objective via nested KRR}

Based on the identifiability result in Prop.~\ref{prop:identifiable}, the AUF objective can be structured as an outer conditional expectation (over pre-alteration covariates $\mathbf{u}$) of an inner conditional expectation (over outcome $\mathbf{y}$). 
Hence, the evaluation naturally decomposes into a two-stage estimation problem. 
Let $\psi(\mathbf{h}) \triangleq \mathbb{E}_{\mathbf{y}\sim P(\cdot\mid \mathbf{h})}[ w_\eta(\mathbf{Y}) ]$ denote the inner expectation given $\mathbf{h} = (\mathbf{x}, \mathbf{u}, \mathbf{a})$. 
With observed context $\mathbf{x}$ and specific alteration $\mathbf{a}$, $\psi(\mathbf{h})$ serves as the response variable dependent on $\mathbf{u}$, acting as the target for the outer expectation $\mathbb{E}_{\mathbf{u}\sim P(\cdot \mid \mathbf{x})}$.
This structural dependency naturally motivates our nested approach: we first estimate $\psi(\mathbf{h})$ in outcome RKHS $\mathcal{H}_{k_y}$, and subsequently use these estimates to solve the outer expectation in context RKHS $\mathcal{H}_{k_x}$.

\textbf{Stage 1: Inner Expectation Estimation.} 
We begin by deriving the estimator for the inner component $\psi(\mathbf{h})$. 
While the original desirability indicator $\mathbb{I}(\y\in\mathcal{S})$ is discontinuous and generally does not reside in the RKHS associated with continuous kernels, the proposed smooth surrogate $w_\eta$ possesses sufficient regularity. 
By assuming that $w_\eta$ resides in the RKHS associated with kernel $k_y$, \ie, $w_\eta \in \mathcal{H}_{k_y}$, we can leverage the \textit{reproducing property} $w_\eta(\y) = \langle w_\eta, k_y(\cdot, \y) \rangle_{\mathcal{H}_{k_y}}$ to reformulate the inner expectation. 
Let $\mu_{\Y \mid \h} \in \mathcal{H}_{k_y}$ denote the CME of the distribution $P(\y \mid \h)$, then we have:
\vspace{-0.05in}
\begin{equation}
    \label{eq:inner_cme_result}
    \begin{aligned}
        \mathbb{E}_{\y\sim P(\cdot\mid \h)}\!\left[ w_\eta(\Y) \right] 
        &= \int_{\mathcal{Y}} \langle w_\eta, k_y(\cdot, \y) \rangle_{\mathcal{H}_{k_y}} \, p(\y\, |\, \h) \, d\y \\
        &= \left\langle w_\eta, \int_{\mathcal{Y}} k_y(\cdot, \y) \, p(\y\, |\, \h) \, d\y \right\rangle_{\mathcal{H}_{k_y}} = \langle w_\eta, \mu_{\Y \mid \h} \rangle_{\mathcal{H}_{k_y}}.
    \end{aligned}
\end{equation}
Consequently, the evaluation of the expectation hinges on the accurate estimation of CME $\mu_{\Y\mid \h}$.
To address this, we employ KRR, a non-parametric framework that effectively captures complex non-linear dependencies in RKHS while admitting a computationally efficient closed-form solution~\citep{GrunewalderLGBPP12}.
Given observational dataset $\mathcal{D} = \{(\h_i, \mathbf{y}_i)\}_{i=1}^N$, where $\h_i$ collects the historical covariates and $\mathbf{y}_i$ denotes the outcome, the CME $\mu_{\Y\mid \h}$ can be empirically estimated as:
\vspace{-0.05in}
\begin{equation}
    \begin{aligned}
        &\hat{\mu}_{\Y\mid \h} = \sum_{i=1}^N \beta_i(\h) k_y(\cdot, \mathbf{y}_i),  \quad \text{with} \quad \boldsymbol{\beta}(\h) = (\K_{hh} + N \lambda_{h} \I)^{-1} \mathbf{k}_{h}(\h),
    \end{aligned}
    \label{eq:cme_estimator}
\end{equation}
where $\mathbf{k}_{h}(\h) = [k_h(\h_1, \h), \dots, k_h(\h_N, \h)]^\T$, $\lambda_{h} > 0$ is the regularization parameter in KRR, and $\K_{hh} \in \mathbb{R}^{N \times N}$ is the Gram matrix with entries $[\K_{hh}]_{ij} = k_h(\h_i, \h_j)$.
Substituting the estimated CME from Eq.~\eqref{eq:cme_estimator} back into the inner product formulation in Eq.~\eqref{eq:inner_cme_result} and let $\mathbf{w}_{\eta} = [w_\eta(\mathbf{y}_1), \dots, w_\eta(\mathbf{y}_N)]^\T$, it can be derived that:
\vspace{-0.03in}
% By leveraging the linearity of the inner product and applying the reproducing property once again, the calculation simplifies to a weighted sum of the desirability values of historical samples:
\begin{equation}
    \begin{aligned}
        \hat{\mathbb{E}}_{\mathbf{y}\sim P(\cdot\mid \mathbf{h})}\!&\left[ w_\eta(\mathbf{Y}) \right]  = \left\langle w_\eta, \sum_{i=1}^N \beta_i(\mathbf{h}) k_y(\cdot, \mathbf{y}_i) \right\rangle_{\mathcal{H}_{k_y}} \\
        &= \sum_{i=1}^N \beta_i(\mathbf{h}) \langle w_\eta, k_y(\cdot, \mathbf{y}_i) \rangle_{\mathcal{H}_{k_y}} = \mathbf{w}_{\eta}^\T (\mathbf{K}_{hh} + N \lambda_{h} \mathbf{I})^{-1} \mathbf{k}_{h}(\mathbf{h}).
    \end{aligned}
    \label{eq:inner_expectation_scalar}
\end{equation}
\textbf{Stage 2: Outer Expectation Estimation.} 
We now proceed to the outer expectation over the pre-alteration covariates $\mathbf{u}$.
Note that the term $\bm{\alpha}^\T \triangleq \mathbf{w}_{\eta}^\T (\mathbf{K}_{hh} + N \lambda_{h} \mathbf{I})^{-1}$ in Eq.~\eqref{eq:inner_expectation_scalar} is a constant row vector determined solely by the observational training data. 
To facilitate the analytical integration over $\mathbf{u}$, we introduce Ass.~\ref{ass:decompose}, a mild structural condition on the kernel that is satisfied by a lot of standard kernels, including the Gaussian RBF and polynomial kernels.
\begin{assumption}[Decomposable Kernels]
    \label{ass:decompose}
    We assume that the kernel $k_h(\cdot, \cdot)$ defined on the joint space $\mathbf{h} = (\mathbf{x}, \mathbf{u}, \mathbf{a})$ admits a multiplicative decomposition: 
    \[k_h(\mathbf{h}_i, \mathbf{h}) = k_x(\mathbf{x}_i, \mathbf{x}) \cdot k_u(\mathbf{u}_i, \mathbf{u}) \cdot k_a(\mathbf{a}_i, \mathbf{a}).\]
\end{assumption}
\vspace{-0.05in}
With this assumption, we substitute the estimator $\hat{\psi}(\mathbf{h})$ back into the outer expectation. 
Since the integration is performed solely \wrt $\mathbf{u}$, the terms of $\mathbf{x}$ and $\mathbf{a}$ remain constant within the integral, thus \(\mathbb{E}_{\mathbf{u}\sim P(\cdot \mid \mathbf{x})}[\hat{\psi}(\mathbf{x}, \mathbf{u}, \mathbf{a})]\) can be computed as:
\begin{equation}
    \begin{aligned}
        & \int_{\sU} \left( \sum_{i=1}^N \alpha_i k_x(\mathbf{x}_i, \mathbf{x}) k_u(\mathbf{u}_i, \mathbf{u}) k_a(\mathbf{a}_i, \mathbf{a}) \right) p(\mathbf{u} \mid \mathbf{x}) \, d\mathbf{u} \\
        =& \sum_{i=1}^N \alpha_i \underbrace{k_x(\mathbf{x}_i, \mathbf{x})}_{\text{Context Match}} \underbrace{k_a(\mathbf{a}_i, \mathbf{a})}_{\text{Action similarity}} \underbrace{\int_{\sU} k_u(\mathbf{u}_i, \mathbf{u}) p(\mathbf{u} \mid \mathbf{x}) \, d\mathbf{u}}_{\text{Expected Unactionable Terms}}.
    \end{aligned}
    \label{eq:outer_decomposition}
\end{equation}
% The integral term in Eq.~\eqref{eq:outer_decomposition} corresponds to the expectation of the feature map $k_u(\mathbf{u}_i, \cdot)$ under the conditional distribution $p(\mathbf{u}\, \mid\, \mathbf{x})$. 
Similar to Eq.~\eqref{eq:cme_estimator}, the second KRR estimator of CME $\mu_{\U \mid \mathbf{x}}$ can be given as \(\hat{\mu}_{\U \mid \mathbf{x}} = \sum_{j=1}^N \gamma_j(\mathbf{x}) k_u(\cdot, \mathbf{u}_j)\), with \(\bm{\gamma}(\mathbf{x}) = (\mathbf{K}_{xx} + N \lambda_{x} \mathbf{I})^{-1} \mathbf{k}_{x}(\mathbf{x})\).
Leveraging this estimator, we can explicitly expand \(\int_{\sU} k_u(\mathbf{u}_i, \mathbf{u}) p(\mathbf{u} \mid \mathbf{x}) \, d\mathbf{u}\) in \Eq~\eqref{eq:outer_decomposition} as:
\begin{equation}
    \begin{aligned}
        \int_{\sU} k_u(\mathbf{u}_i, \mathbf{u})  p(\mathbf{u} \mid \mathbf{x}) \, & d\mathbf{u}
        = \left\langle k_u(\mathbf{u}_i, \cdot), \hat{\mu}_{\U \mid \mathbf{x}} \right\rangle_{\mathcal{H}_{k_u}} \\
        =& \left\langle k_u(\mathbf{u}_i, \cdot), \sum_{j=1}^N \gamma_j(\mathbf{x}) k_u(\cdot, \mathbf{u}_j) \right\rangle_{\mathcal{H}_{k_u}} \\
        = &\sum_{j=1}^N \gamma_j(\mathbf{x}) \underbrace{\langle k_u(\mathbf{u}_i, \cdot), k_u(\mathbf{u}_j, \cdot) \rangle_{\mathcal{H}_{k_u}}}_{k_u(\mathbf{u}_i, \mathbf{u}_j)}.
        % &= [\mathbf{K}_{uu} \bm{\gamma}(\mathbf{x})]_i,
    \end{aligned}
    \label{eq:confounder_bridge_derivation}
\end{equation}
Let $\mathbf{K}_{uu} \in \mathbb{R}^{N \times N}$ denote the Gram matrix with entries $[\mathbf{K}_{uu}]_{ij} = k_u(\mathbf{u}_i, \mathbf{u}_j)$, \Eq~\eqref{eq:confounder_bridge_derivation} equals the \(i\)-th row of vector \(\mathbf{K}_{uu} \bm{\gamma}(\mathbf{x})\).
Substituting Eq.~\eqref{eq:confounder_bridge_derivation} back into Eq.~\eqref{eq:outer_decomposition} yields the closed-form nested estimator of original AUF objective:
\begin{equation*}
    \label{eq:final_scalar_objective}
    \begin{aligned}
        \hat{\bbE}_{\mathbf{u}\sim P(\cdot \mid \mathbf{x})}[\hat{\mathbb{E}}_{\mathbf{y}\sim P(\cdot\mid \mathbf{h})}\!\left[ w_\eta(\mathbf{Y}) \right]  ]
    = &\sum_{i=1}^N \alpha_i [\mathbf{K}_{uu} \bm{\gamma}(\mathbf{x})]_i k_x(\mathbf{x}_i, \mathbf{x}) k_a(\mathbf{a}_i, \mathbf{a})\\
    = & \left(\bm{\alpha}\odot \mathbf{k}_x(\mathbf{x}) \odot (\mathbf{K}_{uu} \bm{\gamma}(\mathbf{x}))\right)^\T \mathbf{k}_a(\a),
    \end{aligned}
\end{equation*}
where $\odot$ denotes Hadamard product.
This is to say, let \(\bm{\omega}(\mathbf{x})=\left(\bm{\alpha}\odot \mathbf{k}_x(\mathbf{x}) \odot (\mathbf{K}_{uu} \bm{\gamma}(\mathbf{x}))\right)\), the final surrogate of AUF objective simplifies to a linear function of the alteration's feature map in RKHS:
\begin{equation}
    \hat{J}_\eta(\mathbf{a}; \mathbf{x}) = \bm{\omega}(\mathbf{x})^\T \mathbf{k}_a(\mathbf{a}).
    \label{eq:final_vector_objective}
\end{equation}
This formulation enables optimizing the proposed alteration $\mathbf{a}$ by maximizing its weighted alignment with the observed realizations of actionable variables in the training set.
The weight vector $\bm{\omega}(\mathbf{x})$ integrates the desirability of historical outcomes (via $\bm{\alpha}$) with the relevance of the current context (via $\mathbf{k}_x{\x}$).
Meanwhile, the term $\mathbf{K}_{uu}$ functions as a confounder adjustment bridge: it measures the similarity between expected confounder profile (via $\bm{\gamma}(\x)$) and the actual pre-alteration covariates observed in historical samples.
This mechanism effectively re-weights the observational data to mitigate the distributional shift induced by actions, ensuring that the optimization targets the influence of the decision alteration $\mathbf{a}$ on the target variables $\mathbf{Y}$ given the context $\mathbf{X}$.
% This mechanism effectively re-weights the observational data to mitigate the distributional shift of latent variables, ensuring that the optimization of $\mathbf{a}$ targets the true structural influence rather than spurious correlations found in the training distribution.
Note that in the degenerate case where $\mathcal{U} = \emptyset$, adjustment term $\mathbf{K}_{uu} \bm{\gamma}(\mathbf{x})$ simplifies to an all-one vector, reflecting that no confounder adjustment is required.

\subsection{Gradient-based rehearsal learning}
\label{sec:optimization}

Having derived the closed-form surrogate $\hat{J}_\eta(\mathbf{a}; \mathbf{x})$, our goal is to find the optimal alteration $\mathbf{a}^*$ that maximizes this objective subject to the feasible constraints:
\begin{equation*}
    \mathbf{a}^* = \arg\max_{\mathbf{a}} \hat{J}_\eta(\mathbf{a}; \mathbf{x}) \quad \text{s.t.} \quad \mathbf{a}_{\text{left}} \preccurlyeq \mathbf{a} \preccurlyeq \mathbf{a}_{\text{right}}.
\end{equation*}
To enable efficient gradient-based optimization and ensure consistency across the framework, we specify all kernels involved in our approach (\ie, $k_x, k_u, k_a, k_y$) as Radial Basis Function (RBF) kernels.
For a generic variable $\mathbf{v} \in \{\mathbf{x}, \mathbf{u}, \mathbf{a}, \mathbf{y}\}$, the corresponding kernel is defined as:
\begin{equation*}
    k_v(\mathbf{v}, \mathbf{v}') = \exp\left( - \frac{\| \mathbf{v} - \mathbf{v}' \|_2^2}{2\sigma_v^2} \right),
\end{equation*}
where $\sigma_v$ denotes the bandwidth parameter specific to each variable type.
This specification naturally satisfies the decomposable assumption in Ass.~\ref{ass:decompose} and provides the necessary smoothness for the subsequent gradient derivation. 
Under this specification, the objective function becomes a weighted sum of Gaussian bumps centered at historical actions $\mathbf{a}_i$:
\begin{equation*}
    \hat{J}_\eta(\mathbf{a}; \mathbf{x}) = \sum_{i=1}^N \omega_i(\mathbf{x}) \exp\left( - \frac{\| \mathbf{a}_i - \mathbf{a} \|_2^2}{2\sigma_a^2} \right).
\end{equation*}
Recall that the weight $\omega_i(\mathbf{x}) = [\bm{\alpha}\odot \mathbf{k}_x(\mathbf{x}) \odot (\mathbf{K}_{uu} \bm{\gamma}(\mathbf{x}))]_i$ encapsulates the confounding-adjusted desirability of the $i$-th historical sample given the current context $\mathbf{x}$.
Crucially, unlike coefficients in standard density estimation, $\omega_i(\mathbf{x})$ can take negative values as it stems from the analytical KRR solutions involving matrix inversions.
% Consequently, the resulting objective landscape is non-convex: positive weights induce an ``attraction'' towards beneficial historical actions, while negative weights generate a ``repulsion'' force away from samples that are either undesirable or structurally biased.

Since the surrogate objective is differentiable, we can derive the gradient with respect to $\mathbf{a}$ analytically. By applying the chain rule, the gradient is given by:
\begin{equation}
    \label{eq:gradient}
    \begin{aligned}
        \nabla_{\mathbf{a}} \hat{J}_\eta(\mathbf{a}; \mathbf{x}) &= \sum_{i=1}^N \omega_i(\mathbf{x}) \nabla_{\mathbf{a}} k_a(\mathbf{a}_i, \mathbf{a}) = \frac{1}{\sigma_a^2} \sum_{i=1}^N \left[ \omega_i(\mathbf{x}) k_a(\mathbf{a}_i, \mathbf{a}) \right] (\mathbf{a}_i - \mathbf{a}).
    \end{aligned}
\end{equation}

\begin{algorithm}[tb]
    \caption{Rehearsal Learning via Nested Estimator}
    \label{alg:nested_krr}
 \begin{algorithmic}[1]
    \STATE {\bfseries Input:} Observational data $\mathcal{D}$, context $\mathbf{x}$, region $\mathcal{S}$.
    \STATE {\bfseries Hyperparameters:} Kernel params $\{\sigma_h, \sigma_x, \sigma_u, \sigma_a\}$, regularization $\{\lambda_h, \lambda_x\}$, smoothing $\eta$, learning rate $\nu$.
    
    \STATE \textit{\gray{// Phase 1: Offline Learning (One-time)}}
    \STATE Compute Gram matrices $\mathbf{K}_{hh}, \mathbf{K}_{xx}, \mathbf{K}_{uu}$.
    \STATE Compute smoothed desirability vector $\mathbf{w}_{\eta}$ using Eq.~\eqref{eq:probit}.
    \STATE Solve stage 1 coefs: $\bm{\alpha}^\T \leftarrow \mathbf{w}_{\eta}^\T (\mathbf{K}_{hh} + N \lambda_{h} \mathbf{I})^{-1}$.
    \STATE Solve stage 2 coefs: $\mathbf{G} \leftarrow (\mathbf{K}_{xx} + N \lambda_{x} \mathbf{I})^{-1}$.
 
    \STATE \textit{\gray{// Phase 2: Context-specific Alteration (Per context $\mathbf{x}$)}}
    \STATE Compute kernel vectors $\mathbf{k}_x(\mathbf{x})$.
    \STATE  $\bm{\gamma}(\mathbf{x}) \leftarrow \mathbf{G} \mathbf{k}_x(\mathbf{x})$.
    \STATE  $\bm{\omega}(\mathbf{x}) \leftarrow \bm{\alpha} \odot \mathbf{k}_x(\mathbf{x}) \odot (\mathbf{K}_{uu} \bm{\gamma}(\mathbf{x}))$.
    
    \STATE \textit{\gray{// Phase 3: Multi-start Optimization}}
    \STATE $\mathcal{A}_{\text{init}} \leftarrow \text{Top-}K \text{ indices } i \text{ where } \omega_i(\mathbf{x}) > 0$.
    \STATE $\mathcal{S}_{\text{sols}} \leftarrow \emptyset$
    \FOR{each $i \in \mathcal{A}_{\text{init}}$}
        \STATE Initialize $\mathbf{a}^{(0)} \leftarrow \mathbf{a}_i$.
        \FOR{$t = 1$ to $T$}
            \STATE Compute gradient $\nabla_{\mathbf{a}} \hat{J}_\eta(\mathbf{a}^{(t-1)}; \mathbf{x})$ via Eq.~\eqref{eq:gradient}.
            \STATE Update: $\mathbf{a}' \leftarrow \mathbf{a}^{(t-1)} + \nu \nabla_{\mathbf{a}} \hat{J}_\eta$.
            \STATE Project: $\mathbf{a}^{(t)} \leftarrow \mathcal{P}_{\{\mathbf{a}|\mathbf{a}_{\text{left}} \preccurlyeq \mathbf{a} \preccurlyeq \mathbf{a}_{\text{right}}\}}(\mathbf{a}')$. 
            % \COMMENT{$\Omega= \{\mathbf{a}_{\text{left}} \preccurlyeq \mathbf{a} \preccurlyeq \mathbf{a}_{\text{right}}\}$}
        \ENDFOR
        \STATE Store solution: $\mathcal{S}_{\text{sols}} \leftarrow \mathcal{S}_{\text{sols}} \cup \{ \mathbf{a}^{(T)} \}$.
    \ENDFOR
    \STATE {\bfseries Output:} $\mathbf{a}^* = \arg\max_{\mathbf{a} \in \mathcal{S}_{\text{sols}}} \hat{J}_\eta(\mathbf{a}; \mathbf{x})$.
 \end{algorithmic}
 \end{algorithm}
 This gradient interpretation is intuitive: it acts as a force vector that pulls the candidate $\mathbf{a}$ towards historical actions $\mathbf{a}_i$ that have high desirability weights $\omega_i(\mathbf{x})$, while being repelled by those with negative weights.
 Meanwhile, objective $\hat{J}_\eta(\mathbf{a}; \mathbf{x})$ presents a potentially non-convex landscape since structured as a sum of Gaussians with mixed-sign weights $\omega_i(\mathbf{x})$, thus random initialization strategy risks entrapment in poor local optima.
 To maximize the estimated objective, we employ a multi-start strategy initialized at historical actions that exhibit high estimated utility.
 Specifically, we construct the candidate set $\mathcal{A}_{\text{init}}$ by selecting historical $\mathbf{a}_i$ whose corresponding weight values $\omega_i(\mathbf{x})$ are among the top-$K$ (up to total number of) positive entries in $\bm{\omega}(\mathbf{x})$:
\begin{equation*}
    \mathcal{A}_{\text{init}} = \{ \mathbf{a}_i \in \mathcal{D} \mid \omega_i(\mathbf{x}) > 0 \text{ and is among top-}K \}.
\end{equation*}
Alg.~\ref{alg:nested_krr} outlines the full decision-making procedure.
It starts with an offline stage where kernel Gram matrices and corresponding KRR coefficients are computed from historical data. 
Given a newly observed context $\mathbf{x}$, the algorithm then computes the context-dependent mixing weights $\bm{\omega}(\mathbf{x})$, which are subsequently used to initialize a multi-start projected gradient ascent procedure that iteratively performs gradient ascent and projects the updates onto the feasible set, eventually yielding the optimal alteration $\mathbf{a}^*$. 
The overall time complexity is \(\sO(N^3)\), dominated by the matrix inversion required in the offline stage (note that it is distinct from the online inference time).

%% file: sections/theory.tex
\subsection{Theoretical analysis}
\label{sec:theory}

% In this subsection, we provide a rigorous theoretical analysis of the proposed framework, covering both the approximation guarantees of the smooth surrogate and the convergence properties of the nested estimator.

We first establish the approximation bound for the smooth desirability function as follows.

\begin{theorem}
    \label{thm:surr}
    Let \(\sS\) be the convex polytope region defined in Eq.~\eqref{eq:region}, and let \(w_\eta(\y)\) be the Probit surrogate of the indicator function \(\bbI(\y\in \sS)\) as defined in Eq.~\eqref{eq:probit}. Assuming the conditional density is bounded, the gap between the true AUF probability and the surrogate objective satisfies:
    \[
    \Big| \bbP(\Y\in\sS\mid \x, \mathring{\a}) - \bbE_{\y\sim P\left(\cdot \mid \x, \mathring{\a}\right)}\left[w_\eta(\y)\right]\Big| = \sO\left( \frac{\sqrt{\ln \eta}}{\eta} \right).
\]
\end{theorem}

Thm.~\ref{thm:surr} provides a theoretical justification for utilizing the Probit surrogate. 
The established bound guarantees that the discrepancy is tightly controlled by the scaling parameter $\eta$, vanishing at a rate of $\tilde{\sO}(1/\eta)$ as $\eta \to \infty$. 
While a larger $\eta$ yields a tighter approximation of the indicator function, it is worth noting that $\eta$ cannot be arbitrarily large in practice. 
As $\eta \to \infty$, the surrogate $w_\eta(\cdot)$ approaches a discontinuity, causing its RKHS norm to grow significantly, in which case an excessively large $\eta$ would require a prohibitive number of observational samples to estimate accurately. 
Thus, $\eta$ serves as a trade-off parameter balancing the approximation bias and the learnability (estimation variance) of the objective.
Next, we characterize statistical performance of our proposed nested estimator $\hat{J}_\eta(\a; \x)$ given a fixed $\eta$.

\begin{theorem}
    \label{thm:nested_bound}
    Let $N$ be the sample size of the observational training data, and let $\lambda_h, \lambda_x > 0$ denote the regularization parameters for the inner (Stage 1) and outer (Stage 2) KRR estimators. 
    Suppose that $w_\eta$ is bounded in the RKHS norm, i.e., $\|w_\eta\|_{\mathcal{H}_{k_y}} \leq C_w$. 
    For any context $\mathbf{x}$ and alteration $\mathbf{a}$, the estimation error $\Delta_{\text{est}} \triangleq \left| \hat{J}_\eta(\mathbf{a}; \mathbf{x}) - J_\eta(\mathbf{a}; \mathbf{x}) \right|$ satisfies:
    \begin{equation*}
        \Delta_{\text{est}} = \mathcal{O}_p\left( \sqrt{\lambda_x} + \sqrt{\lambda_h} + \frac{1}{\sqrt{N \lambda_x}} + \frac{1}{\sqrt{N \lambda_h^3}} \right).
    \end{equation*}
\end{theorem}

Thm.~\ref{thm:nested_bound} explicitly characterizes the learning speed of our nested estimator. 
This bound implies the consistency of the proposed method: provided that the regularization parameters decay appropriately (\ie, $\lambda_x, \lambda_h \to 0$) while satisfying $N \lambda_x \to \infty$ and $N \lambda_h^3 \to \infty$ as $N \to \infty$, the estimation error converges to zero in probability.
For instance, a valid configuration satisfying these conditions is $\lambda_x = \sO(N^{-1/2})$ and $\lambda_h = \sO(N^{-1/4})$.
With appropriate assumptions on the joint distribution of \(\X\), \(\U\), \(\A\) and \(\Y\), better rates can be obtained~\citep{GrunewalderLGBPP12}.

Finally, as Thm.~\ref{thm:surr} characterizes the approximation error of the smoothed surrogate objective and Thm.~\ref{thm:nested_bound} bounds the estimation error induced by finite samples, these results together establish that, with appropriately chosen smoothing hyperparameters and regularization hyperparameters, the established learnable estimator $\hat{J}_\eta$ provides a reliable proxy for the original intractable AUF objective.

%% file: sections/experiments.tex
\section{Experiments}
\label{sec:experiment}
% In this section, we evaluate our approach against rehearsal learning baselines. Experiments were conducted on an NVIDIA Tesla A100 GPU and two Intel Xeon Platinum 8358 CPUs, while specific details are listed in Appx.~\ref{app:experiment}.
% We first briefly introduce the experimental datasets as follows. 
In this section, we present empirical results validating the effectiveness of our proposed method.
Detailed data generation procedures and benchmark construction details are provided in Appx.~\ref{app:experiment}.

\subsection{Experimental setup}
\label{sec:experiment_setup}
We evaluate our method against rehearsal learning baselines, including QWZ23~\citep{qint2023srm}, MICNS~\citep{du2024micns}, CARE~\citep{du2025CARE}, and Grad-Rh~\citep{qin2025gradient}. 
It is worth noting that CARE is theoretically optimal for linear additive Gaussian data, and Grad-Rh is a heuristic approach compatible with non-linear scenarios.
All experiments were conducted on an NVIDIA Tesla A100 GPU and two Intel Xeon Platinum 8358 CPUs.

For the completed benchmark settings, the initial observational dataset size is set to $|\mathcal{D}| = 1,000$, and we report the average performance under $5$ different random seeds.
In our proposed Alg.~\ref{alg:nested_krr}, optimization is initialized with $K=20$ starting points and a learning rate of $5^{-1}$. Kernel bandwidths are determined using the median heuristic~\citep{gretton12a} with minor dataset-specific fine-tuning. The smoothing parameter $\eta$ for the indicator approximation is adaptively selected from $\{5, 10, 20\}$ based on the proportion of positive samples (\ie, \(\y_i\in\sS\)) in $\mathcal{D}$; a larger $\eta$ is employed when positive samples are abundant to allow for sharper approximation.
Lastly, decision quality is evaluated as the probability of successfully avoiding the undesired future, calculated via the empirical success frequency over $100$ Monte Carlo trials sampled from the ground-truth rehearsal distribution.

\subsection{Datasets}
\label{sec:experiment_datasets}
\noindent \textbf{Linear settings.}
We utilize the Bermuda dataset~\citep{courtney2017environmental, bermuda_dataset} that records environmental variables in the Bermuda area, with the objective of maintaining a high net coral ecosystem calcification (NEC) level. 
The data-generating graph structure is known~\citep{courtney2017environmental}, and the parameters are derived by fitting linear models. The dimensions of $\X$, $\Z$, and $\Y$ are 3, 7, and 1, respectively, and the desired region is defined as $\sS=\{\text{NEC}\in[0.5, 2]\}$.
Additionally, we employ the synthetic dataset consistent with~\citet{du2025CARE}, where the dimensions of $\X$, $\Z$, and $\Y$ are 2, 4, and 2, respectively, and the desired region $\sS$ for $\Y$ is a rectangle area.

\noindent \textbf{Nonlinear settings.} 
We first extract a real-data-derived semi-synthetic diabetes benchmark from NHANES 2011--2018~\citep{nhanes_cdc}, using continuous actionable lifestyle / metabolic variables and two glycemic outcomes, HbA1c and fasting plasma glucose (FPG). 
Its desired region requires both outcomes to fall within normal glycemic ranges, $\sS_{\text{NHANES}}=\{(\text{HbA1c}, \text{FPG}) \mid \text{HbA1c}<5.7,\ \text{FPG}<100~\text{mg/dL}\}$, where HbA1c reflects longer-term blood-glucose exposure and FPG measures fasting blood glucose; together they characterize clinically meaningful glycemic status.
We also implement the motivating loan approval example illustrated in \Fig~\ref{fig:auf_example}. To facilitate visualization, we focus on altering solely on the interest rate ($A_2$). The objective is to ensure a high repayment rate ($Y_1$) while maintaining a favorable return on ROI ($Y_2$), thus the desired region is defined as $\sS = \{ (Y_1, Y_2) \mid Y_1 \geq 0.6, Y_2 \geq 0.3 \}$. 
Furthermore, to evaluate performance under complex conditions, we implement several synthetic datasets featuring non-linear mechanisms (including square, trigonometric, and Sigmoid transformations) and non-Gaussian randomness (including uniform, beta, and exponential distributions). 
Appx.~\ref{app:nhanes_benchmark} provides the construction details of the NHANES benchmark.
% Additionally, we assume that the collected data may be right-censored for 3 variables (including $\Y$).

% \begin{figure}[b]
%   \centering
%   \subfigure[$\sS$ of Synthetic data]{\includegraphics[width=0.455\linewidth]{./figures/region1.pdf}\label{fig:region1}}$\quad$
%   \subfigure[$\sS$ of Bermuda data]{\includegraphics[width=0.46\linewidth]{./figures/region2.pdf}\label{fig:region2}}
%   \vskip -0.15in
%   \caption{Illustration for desired regions of two datasets}
%   \label{fig:desired_region}
%   \vskip -0.1in
% \end{figure}

% QWZ23~\citep{qint2023srm} & MICNS~\citep{du2024micns} & Grad-Rh~\citep{qin2025gradient}

\subsection{Results}
Experimental results are presented in Tab.~\ref{tab:main_results}. Baseline results denote the natural AUF probability, \ie, $\mathbb{P}(\mathbf{Y}\in\sS\mid\x)$, observed without any alterations.
In the completed linear additive Gaussian settings, our method outperforms all rehearsal learning baselines except CARE. This performance gap is expected, as CARE is theoretically proven to be optimal under these specific conditions. However, it is crucial to highlight that CARE (and most other baselines, excluding Grad-Rh) know the true parametric family in this case, in contrast, our approach is non-parametric thus does not use prior knowledge of the underlying functional forms.

In the non-linear lending example and synthetic settings, our method demonstrates superior adaptability and robustness, outperforming all comparison methods.
On NHANES, all adapted action methods improve over the no-action baseline on average under the fitted NHANES generator, and our tuned non-parametric adapter obtains the highest mean AUF probability.
This result highlights the potential of rehearsal learning for real-world medical decision problems.

\label{sec:experiment_results}
\begin{table}[t]
    \centering
    \scriptsize % 整体使用小字号，与模板一致
    \caption{The AUF probability $\mathbb{P}(\mathbf{Y} \in \mathcal{S} \mid \mathbf{x}, \mathring{\mathbf{a}})$ evaluated on six benchmark settings. Each completed value is estimated using 100 Monte Carlo samples, averaged over 5 random seeds. Results are reported as mean $\pm$ standard deviation. The best completed results are highlighted in bold.}
    \label{tab:main_results}
    % 使用 resizebox 自动缩放表格到页面宽度的 96%
    \resizebox{\textwidth}{!}{%
    \begin{tabular}{lcccccc}
        \toprule
        \rule{0pt}{1.2ex} % 增加表头行高
        Dataset & Baseline (No action) & QWZ23 & MICNS & CARE & Grad-Rh & Ours \\
        \midrule
        
        % Bermuda
        \rule{0pt}{1.2ex} % 增加行高
        Bermuda & 0.222 {\tiny$\pm$ 0.032} & 0.676 {\tiny$\pm$ 0.027} & 0.668 {\tiny$\pm$ 0.079} & \textbf{0.706 {\tiny$\pm$ 0.065}} 
        & 0.674 {\tiny$\pm$ 0.030}  & 0.702 {\tiny$\pm$ 0.062} \\
        
        % Linear-Syn
        \rule{0pt}{1.2ex}
        Lin-Syn1 & 0.166 {\tiny$\pm$ 0.160} & 0.938 {\tiny$\pm$ 0.032} & 0.912 {\tiny$\pm$ 0.038} 
        & \textbf{0.956 {\tiny$\pm$ 0.019}} & 0.934 {\tiny$\pm$ 0.022} & 0.942 {\tiny$\pm$ 0.016} \\
        
        \midrule % 分隔线

        % NHANES
        \rule{0pt}{1.2ex}
        NHANES & 0.402 {\tiny$\pm$ 0.257} & 0.562 {\tiny$\pm$ 0.217} & 0.536 {\tiny$\pm$ 0.195} 
        & 0.526 {\tiny$\pm$ 0.211} & 0.426 {\tiny$\pm$ 0.203} & \textbf{0.596 {\tiny$\pm$ 0.179}} \\
        
        % Bank
        \rule{0pt}{1.2ex}
        BankExp & 0.598 {\tiny$\pm$ 0.066} & 0.563 {\tiny$\pm$ 0.019} & 0.576 {\tiny$\pm$ 0.014} 
        & 0.640 {\tiny$\pm$ 0.054} & 0.698 {\tiny$\pm$ 0.055} & \textbf{0.820 {\tiny$\pm$ 0.078}} \\
        
        % Dataset 5
        \rule{0pt}{1.2ex}
        Non-Syn1 & 0.064 {\tiny$\pm$ 0.014} & 0.130 {\tiny$\pm$ 0.083} & 0.114 {\tiny$\pm$ 0.091} 
        & 0.078 {\tiny$\pm$ 0.058} & 0.291 {\tiny$\pm$ 0.107} & \textbf{0.430 {\tiny$\pm$ 0.061}} \\

        % Non-Syn
        \rule{0pt}{1.2ex}
        Non-Syn2 & 0.186 {\tiny$\pm$ 0.068} & 0.320 {\tiny$\pm$ 0.237} & 0.324 {\tiny$\pm$ 0.196} 
        & 0.392 {\tiny$\pm$ 0.244} & 0.362 {\tiny$\pm$ 0.243} & \textbf{0.584 {\tiny$\pm$ 0.086}} \\
        
        % Dataset 6
        % \rule{0pt}{1.2ex}
        % Non-Syn3 & 0.000 {\tiny$\pm$ 0.000} & 0.000 {\tiny$\pm$ 0.000} & 0.000 {\tiny$\pm$ 0.000} 
        % & 0.000 {\tiny$\pm$ 0.000} & 0.000 {\tiny$\pm$ 0.000}  & \textbf{0.000 {\tiny$\pm$ 0.000}} \\
        
        \bottomrule
    \end{tabular}%
    }
    % \vspace{0.1in}
\end{table}

\begin{figure}[t] % [t] 参数要求 LaTeX 尝试将其置顶
\centering
    \includegraphics[width=\linewidth]{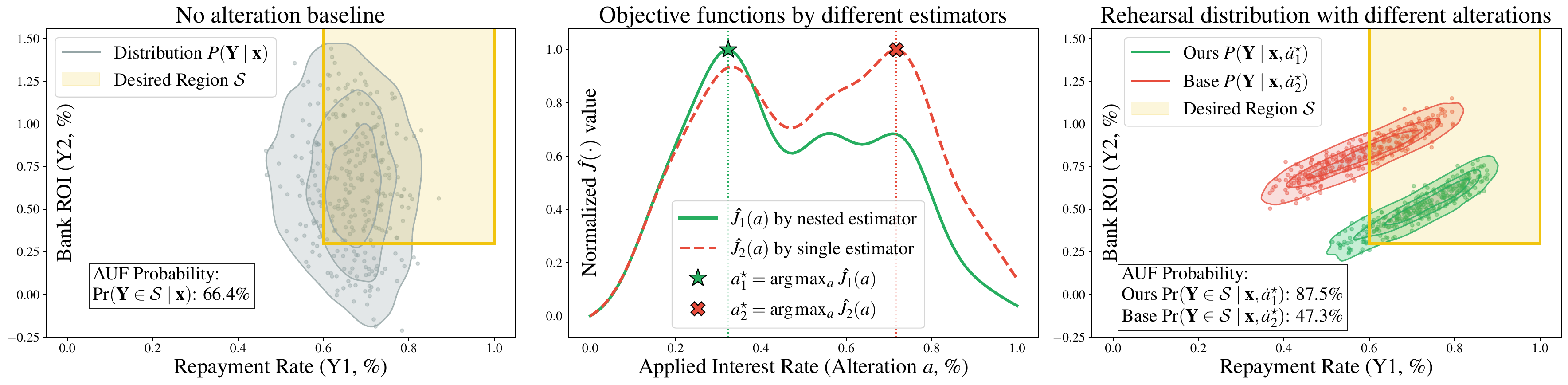}
    \caption{Illustrative example of loan interest rate optimization as shown in Fig.~\ref{fig:auf_example}. The left panel shows the distribution conditioned on context \(\x\) without active alterations. The middle panel displays the comparison of objective functions estimated by single estimator and nested estimator. The right panel illustrates the rehearsal distributions under the alterations optimized by the two estimators.}
\label{fig:auf_bank} % 总图的引用标签
% \vskip -0.03in
\end{figure}

To illustrate the necessity of optimizing the alterational objective $\mathbb{E}_{\mathbf{y} \sim P(\cdot \mid \mathbf{x}, \mathring{\mathbf{a}})}[w_\eta(\mathbf{Y})]$ (identifiable as in Prop.~\ref{prop:identifiable}), rather than directly maximizing $\mathbb{E}_{\mathbf{y} \sim P(\cdot \mid \mathbf{x}, \mathbf{a})}[w_\eta(\mathbf{Y})]$ (pure conditional), we present a comparative analysis on the lending task in Fig.~\ref{fig:auf_example}. Initially, distribution $P(\mathbf{Y}\, |\, \mathbf{x})$ shows a baseline probability of 66.4\% within the desired region $\mathcal{S}$ without active alterations. 
Driven by the goal to actively perform feasible alterations to increase the probability of outcomes falling into $\mathcal{S}$, we compare optimal alterations suggested by maximizing the two estimators. 
Maintaining consistent hyperparameters (\eg, kernel bandwidth), the landscape estimated by single estimator (constructed to estimate $\hat{\mathbb{E}}_{\mathbf{y} \sim P(\cdot \mid \mathbf{x}, \mathbf{a})}[w_\eta(\mathbf{Y})]$) is misled by the unactionable industry stability $U$ that influences both \(\A\) and \(\Y\), suggesting a suboptimal action $a_2^*$. In contrast, our nested estimator explicitly eliminates this confounding, correctly identifying the true optimal alteration $a_1^*$. 
Consequently, the resulting rehearsal distributions confirm that the misled alteration yields a detrimental effect (AUF probability 47.3\%), performing even worse than the no-alteration baseline. Conversely, our method successfully shifts outcomes into $\mathcal{S}$ with an AUF probability of 87.5\%, demonstrating the critical necessity of modeling $\mathbb{E}_{\mathbf{y} \sim P(\cdot \mid \mathbf{x}, \mathring{\mathbf{a}})}[w_\eta(\mathbf{Y})]$ by our nested estimator.

%% file: sections/limitations_conclusion.tex
\section{Conclusion}
\label{sec:conclusion}

In this paper, we presented the first non-parametric rehearsal learning approach for the AUF problem, addressing the key limitations of existing rehearsal-based methods that rely on restrictive parametric assumptions.
By avoiding specific functional forms of the data generation process, our approach extends rehearsal learning to nonlinear and non-additive settings.
Methodologically, we introduced a smooth Probit approximation to handle the discontinuous desirability indicator and developed a KRR-based nested estimator for the surrogate AUF objective.
This estimator enables direct and effective learning of the AUF objective from observational data, grounded in the theory of conditional mean embeddings.
Theoretically, we established identifiability of the AUF objective and derived error bounds for both the smooth approximation and statistical estimation, guaranteeing consistency of the proposed approach.
Empirically, we evaluated the method on synthetic and real-data-derived semi-synthetic benchmarks and added a NHANES diabetes benchmark to stress-test continuous actionable variables and two-dimensional clinical success regions.

% Future work may focus on scaling the proposed method to large-scale datasets using kernel approximation techniques, such as Nystr\"om methods or random Fourier features, and exploring its integration with deep kernel learning for high-dimensional inputs.
% \clearpage
% \newpage
% \section*{Impact Statement}
% This paper presents work whose goal is to advance the field of Machine
% Learning. There are many potential societal consequences of our work, none
% which we feel must be specifically highlighted here.

%% file: appendix/appendix_main.tex
\onecolumn
\input{appendix/related_work}
\input{appendix/structural_rehearsal_models}
\input{appendix/discussion}
\input{appendix/proofs}
\input{appendix/experimental_details}
\input{appendix/scalability_sensitivity}
\input{appendix/nhanes_benchmark}

%% file: appendix/related_work.tex
\section{Related work}
\label{app:related_works}

\noindent\textbf{Rehearsal learning.}
Rehearsal learning exploits influence relations among variables to support decision-making that optimizes the AUF probability~\citep{zhouzh2022rehearsal}.
Existing rehearsal approaches adopt parametric formulations, relying on assumptions like linear additive systems to derive probabilistic constraints~\citep{qint2023srm,du2024micns,tao2025AUFSD} or tractable objectives~\citep{du2025CARE}. 
\citet{qin2025gradient} extend the parametric paradigm to nonlinear settings using conditional normalizing flows~\citep{winkler2019learning,ardizzone2019guided} to estimate generation paramters, but replace AUF objective with heuristic surrogates. 
While offering tractability, this method deviates from true AUF objective and often leads to suboptimal decisions; Appx.~\ref{app:discuss} gives a concrete counter-example. 
In contrast, our approach directly targets AUF probability without heuristic substitutions.
% We also acknowledge a concurrent submission~\citep{anonymous2026concurrent} (per ICML policy), which primarily focuses on order-learning; its decision-making component follows \citet{qin2025gradient} and thus inherits the same limitation. 

\noindent\textbf{Reinforcement Learning.} 
RL has demonstrated strong empirical performance across various decision-making problems~\citep{sutton2018reinforcement}. 
Conventional RL algorithms~\citep{John2017PPO, haarnoja2018SAC} typically depend on repeated and potentially costly interactions with the environment, rendering them unsuitable for AUF settings where such interactions are expensive and limited~\citep{zhouzh2022rehearsal,qint2023srm}. 
Recent progress in offline and hybrid offline--online RL~\citep{song2022hybrid, pong2022offline, pmlr-v202-li23av, NEURIPS2023_c1aaf7c3} seeks to alleviate this issue by exploiting pre-collected data; 
however, hybrid approaches presume the feasibility of online interaction for policy refinement, which remains inapplicable to high-stakes AUF scenarios and thus cannot handle the confounding bias inherent in passive observational data.
Besides, several model-based RL approaches aim to mitigate confounding effects in offline policy evaluation~\citep{kallus2020confounding, bruns2021model, kausik2024offline}. 
In contrast to our framework (utilizing the order structure without imposing strong assumptions on the confounding), these methods often rely on restrictive structural conditions and assume that confounders influence only transition dynamics. 

\noindent\textbf{Causality.} 
A substantial body of work has investigated the use of structural models for decision-making, with most approaches rooted in the framework of structural causal models (SCMs)~\citep{spirtes2000causation}. 
Several studies focus on system identification~\citep{he2008active, zhang2009, KocaogluDV17, zhang2017causal, cai2018, shiragur2024causal}. 
However, these methods typically emphasize recovering the underlying causal structure itself, rather than addressing downstream objectives such as estimating or optimizing decision-relevant causal effects. 
Additionally, several recent works also combine kernel mean embeddings with causal or policy quantities~\citep{MuandetKSM21, NEURIPS2025_49a77f55}; however, they mainly study how to estimate or evaluate the distribution induced by a fixed treatment or policy, rather than how to choose a feasible context-dependent action. 
In parallel, causal bandit frameworks have been proposed to tackle optimal arm identification problems~\citep{bareinboim2015bandits, latt2016causalbandit, sen2017identifying, lee2018structural, zhang2019RL, NEURIPS2020_61a10e6a, LuMT21, park2025structural, park2025transportabilitystructuralcausalbandits}. 
Such approaches generally seek a single, globally optimal action that maximizes the expected reward. 
By contrast, rehearsal learning methods naturally accommodates mutually influenced relationships among variables and targets the maximization of the AUF probability by selecting context-dependent optimal alterations conditioned on the observed context \(\x\), and allows for a more expressive specification of the desired outcome region \(\sS\), extending beyond simple expectation maximization objectives.

%% file: appendix/structural_rehearsal_models.tex
\section{Additional background on graph-based rehearsal models}
\label{app:graph_based_background}

This appendix recalls the graph-based structural rehearsal model terminology used in prior work.
The main identifiability result in Sec.~\ref{sec:smooth_identifiability} does not assume access to such a full graph; it relies on the known order and the valid decomposition in Def.~\ref{def:Z_decompose}.

The structural rehearsal model is a graphical model that characterizes influence relations for decision-making, consisting of (i) a set of (potentially time-varying) rehearsal graphs; and (ii) their associated generation equations~\citep{qint2023srm}. 
Rehearsal graph $G = (\V, \E)$ captures the qualitative generating relations among variables, where each vertex corresponds to a variable in the decision task. 
Notably, rehearsal graphs employ bi-directional edges to connect variables that are mutually influenced. 
The formal definitions are provided below following~\citet{qint2023srm}.

\begin{definition}[Mixed graph]
    \label{def:mixed_graph}
    Let $G=(\V, \E)$ be a graph, where $\V$ and $\E$ denote the vertices and the edges. $G$ is a \emph{mixed graph} if for any distinct vertices $u, v \in \V$, there is at most one edge connecting them, and the edge is either \emph{directional} ($u \rightarrow v$ or $u \leftarrow v$) or \emph{bi-directional} ($u \leftrightarrow v$).
\end{definition}

\begin{definition}[Bi-directional clique]
    \label{def:bi_clique}
    A \emph{bi-directional clique} $C=\left(\V^c, \E^c\right)$ of a mixed graph $G=(\V, \E)$ is a complete subgraph induced by $\V^c \subseteq \V$ such that any edge $e \in \E^c$ is bi-directional. $C$ is \emph{maximal} if adding any other vertex does not induce a bi-directional clique.
\end{definition}

\begin{definition}[Rehearsal graph]
    \label{def:rh_graph}
    Let $G=(\V, \E)$ be a mixed graph and $\left\{C_i\right\}_{i=1}^l$ denote all maximal bi-directional cliques of $G$, where $C_i=\left(\V_i^c, \E_i^c\right)$. $G$ is a \emph{rehearsal graph} iff: (i) $\V_i^c \cap \V_j^c=\emptyset$ for any $i \neq j$; (ii) for any $i \in[l]$, if there is any edge pointing from some $u \in \V \backslash \V_i^c$ to some $v \in \V_i^c$, then for any $v \in \V_i^c$, $u \rightarrow v$; and (iii) the directional edges permit a topological ordering for $\left\{C_i\right\}_{i=1}^l$.
\end{definition}

Associated with the graphical representation is a set of local conditional distributions, which characterizes the generating process of variables quantitatively. 
Given a rehearsal graph $G$, these distributions are defined over the set of bi-directional cliques $\{C_i\}_{i=1}^l$. 
Let $\PA_i^G \triangleq \{u \mid \exists v \in \V_i^c, u \to v \text{ in } G\}$ denote parents of clique $C_i$, representing external variables that influence the clique. 
The values of variable sets $\V_i^c$s are generated based on their parents $\PA_i^G$ according to a general conditional distribution $P(\V_i^c \mid \PA_i^G)$ and the joint distribution $P(\V)$ factorizes as \( P(\V) = \prod_{i=1}^l P(\V_i^c \mid \PA_i^G)\).

\setlength{\intextsep}{2pt} % 将环绕图片的间距设置为零
\begin{wrapfigure}[10]{r}{0.64\textwidth}
  \centering
  \begin{subfigure}{0.31\linewidth}
      \centering
      \includegraphics[width=\linewidth]{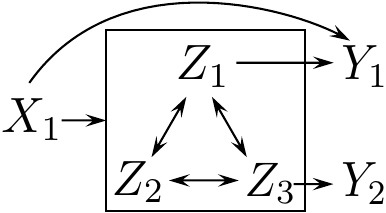}
      \caption{$G$}
      \label{fig:rh_graph_1}
  \end{subfigure}
  \hfill
  \begin{subfigure}{0.31\linewidth}
      \centering
      \includegraphics[width=\linewidth]{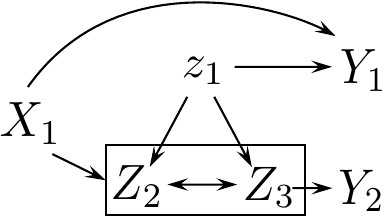}
      \caption{$G^{\mathring{z}_1}$}
      \label{fig:rh_graph_2}
  \end{subfigure}
  \hfill
  \begin{subfigure}{0.31\linewidth}
      \centering
      \includegraphics[width=\linewidth]{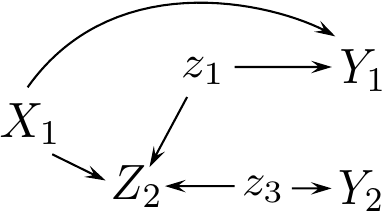}
      \caption{$G^{\mathring{z}_1, \mathring{z}_3}$}
      \label{fig:rh_graph_3}
  \end{subfigure}
  \caption{\ref{fig:rh_graph_1} is an example rehearsal graph, while \ref{fig:rh_graph_2} and \ref{fig:rh_graph_3} are graphs after rehearsal operations \(\mathring{z}_1\) and \(\mathring{z}_1, \mathring{z}_3\), characterizing the data-generating processes after certain alterations are performed.}
  \label{fig:three_graphs}
\end{wrapfigure}

To simulate the results of potential decisions in this graph-based formulation, one can perform rehearsal operation $\mathbf{S} \aleq \s$ (or abbreviated as \(\mathring{\s}\)).
As illustrated in Fig.~\ref{fig:rh_graph_2}, applying $\mathring{\s}$ performs a graph surgery: it severs all incoming edges to decision variables \(\mathbf{S}\) while assigning them specific values $\s$, structurally isolating the action from its natural influences.
This operation blocks spurious correlations induced by confounders, necessitating a sharp distinction between distribution $P(\cdot \mid \mathring{\s})$ and observational conditional $P(\cdot \mid \s)$.

%% file: appendix/discussion.tex
\section{Discussion}
\label{app:discuss}

\setlength{\intextsep}{12pt} % 将环绕图片的间距设置为零
\begin{wrapfigure}[14]{r}{0.55\textwidth}
    \vspace{-15pt}
    \centering
    \includegraphics[width=0.55\textwidth]{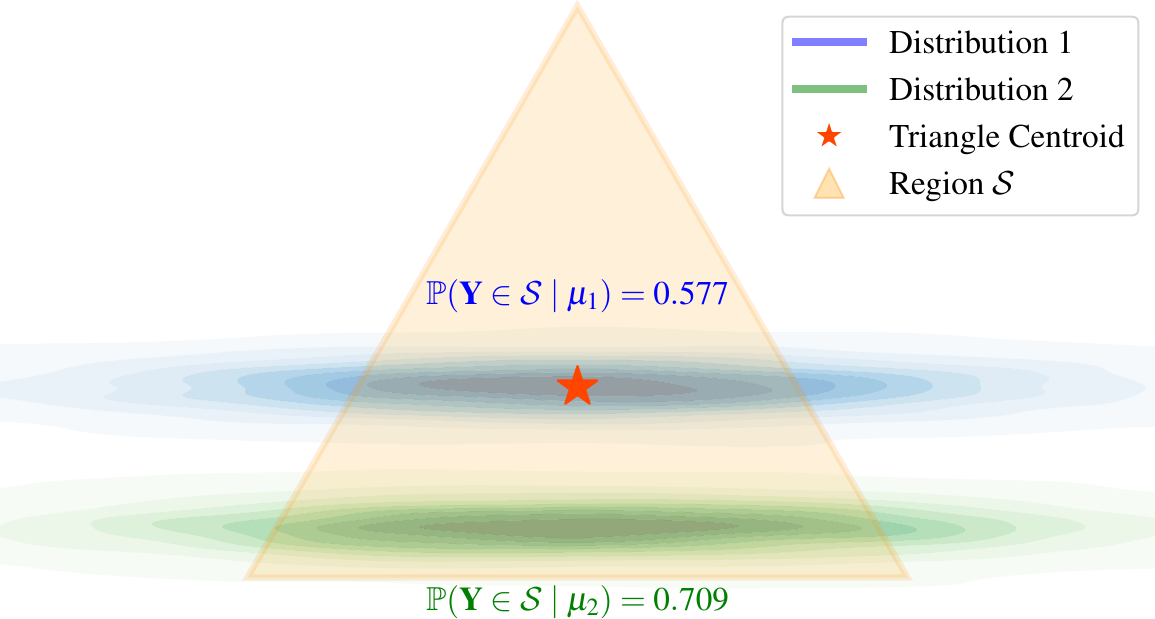}
    % \vspace{-0.5in}
    \caption{A counter-example showing the misalignment between distance minimization and true AUF objective.}
    \label{fig:counter_example}
\end{wrapfigure}
Prior rehearsal learning methodologies have been predominantly confined to the assumption that the underlying generation mechanisms follow linear additive noise models. 
While \citet{qin2025gradient} propose a flow-based rehearsal learning method purportedly adaptable to nonlinear systems, a critical limitation lies in its reliance on a heuristic surrogate objective. 
To circumvent the intractability of directly optimizing the AUF probability, \citet{qin2025gradient} propose identifying the alteration $\mathbf{a}$ that minimizes $\|\bm{\mu}_{\mathbf{Y}|\mathbf{x}, \mathring{\mathbf{a}}} - \mathbf{c}\|_2$, where $\mathbf{c}$ denotes the centroid of region $\mathcal{S}$. However, this heuristic often misaligns with true AUF objective (maximizing $\bbP(\Y\in\sS\mid\mathbf{x}, \mathring{\mathbf{a}})$). As illustrated in Fig.~\ref{fig:counter_example} ($\mathcal{S}$ is an equilateral triangle, a common polytope), the strategy minimizing the heuristic loss (blue) perfectly aligns the distribution mean with the centroid ($\|\bm{\mu}_1 - \mathbf{c}\| = 0$) yet only yields a success probability of $0.577$. Conversely, an alternative strategy (green) explicitly violates the heuristic condition ($\|\bm{\mu}_2 - \mathbf{c}\| > 0$) but achieves a significantly higher probability of $0.709$. This demonstrates that this distance-based surrogate can result in suboptimal decisions, diverging from the goal of AUF.

To surmount this limitation, we leverage CMEs to reconstruct the optimization problem. We begin by approximating the discontinuous indicator function $\bbI(\cdot)$ with a continuous one $w_\eta(\cdot)$, allowsing us to reformulate the maximization of $\bbP(\Y\in\sS\mid\mathbf{x}, \mathring{\mathbf{a}})$ into the maximization of the inner product between $w_\eta$ and the CME of $\mathbf{Y}$ within the RKHS. Subsequently, we nestedly utilize  KRR to estimate the CME, yielding a kernel-based surrogate AUF objective. Theoretically, Thm.~\ref{thm:surr} and Thm.~\ref{thm:nested_bound} guarantee that the deviation between our surrogate and the true AUF objective is bounded, ensuring optimization consistency. A significant challenge inherent in our non-parametric approach is hyperparameter selection, including $\eta$ for smooth approximation, bandwidths of RBF kernels, and regularization parameters in KRR. In this work, hyperparameters were determined empirically, referencing established heuristics in kernel literature~\citep{gretton12a}. A critical avenue for future research is to explore principled methods for hyperparameter selection.

%% file: appendix/proofs.tex
\section{Proofs}
\label{app:proof}
In this section, we provide proof for claims in the main text.

\subsection{Proof of Proposition~\ref{prop:identifiable}}

\begin{prop1}
    Let $\Z = \U \cup \A \cup \U^{\text{post}}$ be the decomposition defined in Def.~\ref{def:Z_decompose}. The AUF objective defined by the smooth desirability approximation $w_\eta(\cdot)$ under the feasible alteration $\mathring{\a}$ is identifiable from the observational distribution as:
    \begin{equation*}
      \bbE_{\y\sim P\left(\cdot \mid \x, \mathring{\a}\right)}\!\!\left[w_\eta(\Y)\right]
      =
      \bbE_{\u\sim P(\cdot \mid \x)}\!\!\left[\bbE_{\y\sim P(\cdot\mid \x, \u, \a)}\!\!\left[ w_\eta(\Y) \right]\right].
    \end{equation*}
\end{prop1}

\begin{proof}[Proof of Prop.~\ref{prop:identifiable}]
    By the sufficient decomposition in Def.~\ref{def:Z_decompose}, $\U$ contains the pre-alteration non-actionable variables sufficient for adjustment given the observed context $\X$.
    Applying the adjustment formula~\citep{pearl2009causality} after conditioning on $\x$, the distribution of $\Y$ under the feasible alteration $\mathring{\a}$ can be identified via the observational distribution:
    \begin{equation*}
        p(\y \mid \x, \mathring{\a}) = \int_\sU p(\y \mid \x, \u, \a) \, p(\u \mid \x) \, d\u.
    \end{equation*}
    The post-alteration variables $\U^{\text{post}}$ are not conditioned on because they occur after the alteration and may lie on downstream pathways from $\A$ to $\Y$; their contribution is marginalized inside $p(\y \mid \x, \u, \a)$.
    Multiplying by $w_\eta(\y)$ and integrating over $\y$ yields the target expectation:
    \begin{equation*}
        \begin{aligned}
            \bbE_{\y\sim P(\cdot \mid \x, \mathring{\a})}[w_\eta(\Y)] 
            &= \int_{\sY}w_\eta(\y)p(\y \mid \x, \mathring{\a})\, d\y\\
            &= \int_{\sY}w_\eta(\y)\left(\int_\mathcal{U}p(\y \mid \x, \u, \a)p(\u\mid \x)\, d\u\right)\, d\y\\
            &=\int_{\mathcal{U}} \left(\int_{\sY}w_\eta(\y) p(\y \mid \x, \u, \a)\, d\y\right)p(\u\mid \x)\, d\u\\
            &= \bbE_{\u\sim P(\cdot \mid \x)}\!\left[\bbE_{\y\sim P(\cdot\mid \x, \u, \a)}\!\left[ w_\eta(\Y) \right]\right].
        \end{aligned}
    \end{equation*}
\end{proof}

\subsection{Proof of Theorem~\ref{thm:surr}}
% We provide the complete proof for Thm.~\ref{thm:surr} below.

\begin{lemma}[\citeauthor{vershynin2018high}, \citeyear{vershynin2018high}, Proposition 2.1.2]
    \label{lemma:tail_gaussian}
    Let $\Phi(x) = \frac{1}{\sqrt{2\pi}} \int_{-\infty}^x e^{\frac{-u^2}{2}} \, du$ denote the CDF of the standard normal distribution. For any $t > 0$, the following inequality holds:
    \begin{equation*}
        \Phi(-t) \leq \frac{1}{\sqrt{2\pi}t} e^{\frac{-t^2}{2}}.
    \end{equation*}
\end{lemma}

\begin{lemma}
    \label{lemma:telescoping}
    Let $\{u_k\}_{k=1}^l$ and $\{v_k\}_{k=1}^l$ be two sequences of real numbers such that $u_k, v_k \in [0, 1]$ for all $k = 1, \dots, l$. Then, the absolute difference of their products is bounded by the sum of their absolute pointwise differences, as shown in:
    \begin{equation*}
        \left| \prod_{k=1}^l u_k - \prod_{k=1}^l v_k \right| \leq \sum_{k=1}^l |u_k - v_k|.
    \end{equation*}
\end{lemma}

\begin{proof}[Proof of Lemma~\ref{lemma:telescoping}]
    The result follows from a standard telescoping sum argument:
    \begin{equation*}
        \prod_{k=1}^l u_k - \prod_{k=1}^l v_k = \sum_{k=1}^l \left( \prod_{i=1}^{k-1} v_i \right) (u_k - v_k) \left( \prod_{j=k+1}^l u_j \right)\leq \sum_{k=1}^l |u_k - v_k|.
    \end{equation*}
    The second inequality holds as $|u_i|, |v_i| \le 1$, thus the magnitude of coefficients is bounded by 1.
\end{proof}

\begin{theorem1}
    Let \(\sS\) be the convex polytope region defined in Eq.~\eqref{eq:region}, and let \(w_\eta(\y)\) be the Probit surrogate of the indicator function \(\bbI(\y\in \sS)\) as defined in Eq.~\eqref{eq:probit}. Assuming the conditional density is bounded, the gap between the true AUF probability and the surrogate objective satisfies:
    \[
    \Big| \bbP(\Y\in\sS\mid \x, \mathring{\a}) - \bbE_{\y\sim P\left(\cdot \mid \x, \mathring{\a}\right)}\left[w_\eta(\y)\right]\Big| = \sO\left( \frac{\sqrt{\ln \eta}}{\eta} \right).
\]
\end{theorem1}

\begin{proof}[Proof of Thm.~\ref{thm:surr}]
    Let the pointwise approximation error be denoted by \(\Delta(\y) = \left| \mathbb{I}(\y \in \sS) - w_\eta(\y) \right|\), where the indicator function can be written as the product of indicators for each linear constraint: \(\mathbb{I}(\y \in \sS) = \prod_{k=1}^l \mathbb{I}(b_k - \m_k^\T\y \ge 0)\). Similarly, the surrogate is the product of Probit functions \(w_\eta(\y) = \prod_{k=1}^l \Phi(\eta(b_k - \m_k^\T\y))\) as defined in Eq.~\eqref{eq:probit}.
    We aim to bound the probability gap. By invoking the triangle inequality for integrals, we decompose the gap into: (i) the boundary tube \(\partial_\epsilon \sS\) (points within distance \(\epsilon\) of boundary); and (ii) the complement region \(\sS^c_{\epsilon}\) (points at least \(\epsilon\) away from boundary).
    \begin{equation}
        \label{eq:gap}
        \begin{aligned}
            \text{Gap} & \triangleq  \left| \bbP(\Y\in\sS\mid \x, \mathring{\a}) - \bbE_{\y\sim P\left(\cdot \mid \x, \mathring{\a}\right)}\left[w_\eta(\y)\right]\right|\\
            & = \left| \int_{\bbR^{d_y}} \left( \mathbb{I}(\y \in \sS) - w_\eta(\y) \right) p(\y \mid \x, \mathring{\a}) \, d\y \right| \\
            & \leq \int_{\bbR^{d_y}} \Delta(\y) p(\y \mid \x, \mathring{\a}) \, d\y \\
            & = \int_{\partial_\epsilon \sS} \Delta(\y) p(\y \mid \x, \mathring{\a}) \, d\y + \int_{\sS^c_{\epsilon}} \Delta(\y) p(\y \mid \x, \mathring{\a}) \, d\y.
        \end{aligned}
    \end{equation}
    
    We analyze the two integral terms separately. Without loss of generality, we assume $\|\m_k\|_2 = 1$ for each \(\m_k^\T\y \leq b_k\) in \Eq~\eqref{eq:region}, as each constraint can be rescaled accordingly. 
    
    \textbf{Bound for the Boundary Region (\(\partial_\epsilon \sS\)).}
    The pointwise error can be bounded by \(\Delta(\y) \le 1\). Assuming the probability density \(p(\y \mid \x, \mathring{\a})\) is bounded by a constant \(M\), the probability mass of the boundary tube depends on the volume of \(\partial_\epsilon \sS\). For a convex polytope, the volume of the \(\epsilon\)-neighborhood of the boundary scales linearly with \(\epsilon\) for sufficiently small \(\epsilon\) \wrt the Euclidean metric. Thus, there exists a constant \(K\) dependent on the surface area of \(\sS\) and \(M\) such that:
    \begin{equation}
        \label{eq:boundary_bound}
        \int_{\partial_\epsilon \sS} \Delta(\y) p(\y \mid \x, \mathring{\a}) \, d\y \leq \int_{\partial_\epsilon \sS} M \, d\y \leq K \epsilon.
    \end{equation}
    
    \textbf{Bound for the Complement Region (\(\sS^c_{\epsilon}\)).}
    In this region, the distance from \(\y\) to any defining hyperplane is at least \(\epsilon\). Let \(h_k(\y) = b_k - \m_k^\T\y\) denote the signed margin for the \(k\)-th constraint. We distinguish two cases for \(\y \in \sS^c_{\epsilon}\) as follows:
    \begin{itemize}[leftmargin=12pt, itemsep=0.03in]
        \item \textit{Deep Inside (\(\y \in \sS\)):} In this case, \(\mathbb{I}(\y \in \sS) = 1\), and \(h_k(\y) \ge \epsilon\) for all \(k\). The exact term \(\mathbb{I}(\y \in \sS) = \prod_{k=1}^l \mathbb{I}(h_k(\y) \ge 0) = 1\), and the surrogate \(w_\eta(\y) = \prod_{k=1}^l \Phi(\eta h_k(\y))\). We apply Lemma~\ref{lemma:telescoping} setting \(u_k = 1\) and \(v_k = \Phi(\eta h_k(\y))\):
        \begin{equation*}
            \Delta(\y) = \left| 1 - \prod_{k=1}^l \Phi(\eta h_k(\y)) \right| \leq \sum_{k=1}^l \left| 1 - \Phi(\eta h_k(\y)) \right| = \sum_{k=1}^l \Phi(- \eta h_k(\y)).
        \end{equation*}
        Since \(h_k(\y) \ge \epsilon\) and \(\Phi(\cdot)\) is monotonic, \(\Phi(-\eta h_k(\y)) \le \Phi(-\eta \epsilon)\). Thus, \(\Delta(\y) \le l \cdot \Phi(-\eta \epsilon)\).
        
        \item \textit{Deep Outside (\(\y \notin \sS\)):} In this case, \(\mathbb{I}(\y \in \sS) = 0\). There exists at least one constraint \(j\) such that \(h_j(\y) \le -\epsilon\). The exact term \(\mathbb{I}(\y \in \sS) = \prod_{k=1}^l \mathbb{I}(h_k(\y) \ge 0) = 0\), and the surrgate \(w_\eta(\y) = \prod_{k=1}^l \Phi(\eta h_k(\y))\). Since \(0 < \Phi(\cdot) < 1\), the product is bounded by any single factor:
        \begin{equation*}
            \Delta(\y) = \prod_{k=1}^l \Phi(\eta h_k(\y)) \le \Phi(\eta h_j(\y)) \le \Phi(-\eta \epsilon).
        \end{equation*}
    \end{itemize}
    
    Combining both cases, the pointwise error in \(\sS^c_{\epsilon}\) is globally dominated by the term involving the Gaussian tail: \(\Delta(\y) \le l \cdot \Phi(-\eta \epsilon)\).
    Now, we apply Lemma~\ref{lemma:tail_gaussian} setting \(t = \eta \epsilon > 0\):
    \begin{equation*}
        \sup_{\y \in \sS^c_{\epsilon}} \Delta(\y) \leq l \cdot \Phi(-\eta \epsilon) \leq l \cdot \frac{1}{\sqrt{2\pi}\eta \epsilon} e^{-\frac{(\eta \epsilon)^2}{2}}.
    \end{equation*}
    The integral over the complement region is therefore bounded by:
    \begin{equation}
        \label{eq:inside_bound}
        \int_{\sS^c_{\epsilon}} \Delta(\y) p(\y \mid \x, \mathring{\a}) \, d\y \le \frac{l}{\sqrt{2\pi}\eta \epsilon} e^{-\frac{\eta^2 \epsilon^2}{2}}.
    \end{equation}
    
    \textbf{Total Bound.} Combining Eq.~\eqref{eq:boundary_bound} and Eq.~\eqref{eq:inside_bound}, the total gap in Eq.~\eqref{eq:gap} is bounded by:
    \begin{equation*}
        \text{Gap} \leq K \epsilon + \frac{C}{\eta \epsilon} e^{-\frac{\eta^2 \epsilon^2}{2}},
    \end{equation*}
    where \(C = l/\sqrt{2\pi}\). Choosing \(\epsilon = \frac{\sqrt{2\ln \eta}}{\eta}\), we have:
    \begin{itemize}[leftmargin=12pt, itemsep=0.03in]
        \item The boundary term becomes \(K \frac{\sqrt{2\ln \eta}}{\eta}\).
        \item The tail term becomes \(\frac{C}{\sqrt{2\ln \eta}} e^{-\ln \eta} = \frac{C}{\sqrt{2\ln \eta}} \frac{1}{\eta}\), which decays faster than the boundary term.
    \end{itemize}
    Consequently, the dominant term determines the convergence rate:
    \begin{equation*}
        \text{Gap} \leq \sO\left( \frac{\sqrt{\ln \eta}}{\eta} \right).
    \end{equation*}
\end{proof}

\subsection{Proof of Theorem~\ref{thm:nested_bound}}
\label{sec:theory_bound}

\begin{lemma}[\citeauthor{song2009hilbert}, \citeyear{song2009hilbert}, Theorem 6]
    \label{lemma:cme_pointwise}
    Let $\mathcal{V}$ and $\mathcal{Q}$ be the domains of the conditioning variable $v$ and the target variable $q$, respectively. Let $\mathcal{H}_{k_q}$ denote the RKHS associated with the target domain.
    Consider the empirical conditional mean embedding $\hat{\mu}_{q|v}$ estimated from $N$ samples using kernel ridge regression (KRR) with regularization parameter $\lambda$. 
    Assuming the kernel $k_v$ on $\mathcal{V}$ is bounded and the true embedding $\mu_{q|v}$ lies in the range of the covariance operator, for a fixed $v \in \mathcal{V}$, the estimation error satisfies the following bound in probability:
    \begin{equation*}
        \| \hat{\mu}_{q|v} - \mu_{q|v} \|_{\mathcal{H}_{k_q}} = \mathcal{O}_p\left( \frac{1}{\sqrt{N \lambda}} + \sqrt{\lambda} \right).
    \end{equation*}
\end{lemma}

\begin{lemma}[\citeauthor{grunewalder2012modelling}, \citeyear{grunewalder2012modelling}, Lemma 2.1]
    \label{lemma:cme_uniform}
    Under standard regularity assumptions on the marginal distribution and kernels defined on $\mathcal{V}$ and $\mathcal{Q}$, the empirical conditional mean embedding converges uniformly over the conditioning domain $\mathcal{V}$. 
    Specifically, with regularization parameter $\lambda$, the following uniform bound holds in probability:
    \begin{equation*}
        \sup_{v \in \mathcal{V}} \| \hat{\mu}_{q|v} - \mu_{q|v} \|_{\mathcal{H}_{k_q}} = \mathcal{O}_p\left( \sqrt{\lambda} + \frac{1}{\sqrt{N \lambda^3}} \right).
    \end{equation*}
\end{lemma}

Based on these lemmas, we establish the consistency and finite-sample error bound for our nested estimator.

% \begin{theorem}[Finite-Sample Error Bound]
%     \label{thm:nested_bound}
%     Let $\lambda_h$ and $\lambda_x$ denote the regularization parameters for the inner (Stage 1) and outer (Stage 2) KRR estimators, respectively. Assume that the smooth surrogate $w_\eta$ is bounded in RKHS norm, i.e., $\|w_\eta\|_{\mathcal{H}_{k_y}} \leq C_w$. 
%     For any query context $\mathbf{x}$ and alteration $\mathbf{a}$, the estimation error satisfies:
%     \begin{equation*}
%         \left| \hat{\mathbb{E}}_{\mathbf{u}\sim P(\cdot \mid \mathbf{x})}\left[\hat{\mathbb{E}}_{\mathbf{y}\sim P(\cdot\mid \mathbf{h})}\!\left[ w_\eta(\mathbf{Y}) \right] \right] - \mathbb{E}_{\mathbf{u}\sim P(\cdot \mid \mathbf{x})}\left[\mathbb{E}_{\mathbf{y}\sim P(\cdot\mid \mathbf{h})}\!\left[ w_\eta(\mathbf{Y}) \right] \right] \right| 
%         = \mathcal{O}_p\left( \sqrt{\lambda_x} + \frac{1}{\sqrt{N \lambda_x}} + \sqrt{\lambda_h} + \frac{1}{\sqrt{N \lambda_h^3}} \right).
%     \end{equation*}
% \end{theorem}
\begin{theorem2}
    Let $N$ be the sample size of the observational training data, and let $\lambda_h, \lambda_x > 0$ denote the regularization parameters for the inner (Stage 1) and outer (Stage 2) KRR estimators. 
    Suppose that $w_\eta$ is bounded in the RKHS norm, i.e., $\|w_\eta\|_{\mathcal{H}_{k_y}} \leq C_w$. 
    For any context $\mathbf{x}$ and alteration $\mathbf{a}$, the estimation error $\Delta_{\text{est}} \triangleq \left| \hat{J}_\eta(\mathbf{a}; \mathbf{x}) - J_\eta(\mathbf{a}; \mathbf{x}) \right|$ satisfies:
    \begin{equation*}
        \Delta_{\text{est}} = \mathcal{O}_p\left( \sqrt{\lambda_x} + \sqrt{\lambda_h} + \frac{1}{\sqrt{N \lambda_x}} + \frac{1}{\sqrt{N \lambda_h^3}} \right).
    \end{equation*}
\end{theorem2}

\begin{proof}[Proof of Thm.~\ref{thm:nested_bound}]
    Let $\psi(\mathbf{h}) \triangleq \mathbb{E}_{\mathbf{y}\sim P(\cdot\mid \mathbf{h})}[ w_\eta(\mathbf{Y}) ]$ denote the true inner expectation function, and $\hat{\psi}(\mathbf{h}) \triangleq \hat{\mathbb{E}}_{\mathbf{y}\sim P(\cdot\mid \mathbf{h})}[ w_\eta(\mathbf{Y}) ]$ denote its empirical estimator obtained from the first-stage KRR. Note that for fixed context $\mathbf{x}$ and alteration $\mathbf{a}$, both can be viewed as functions of $\mathbf{u}$.
    We aim to bound the absolute error denoted by $\Delta$:
    \begin{equation*}
        \Delta = \left| \hat{\mathbb{E}}_{\mathbf{u}\sim P(\cdot \mid \mathbf{x})}[\hat{\psi}] - \mathbb{E}_{\mathbf{u}\sim P(\cdot \mid \mathbf{x})}[\psi] \right|.
    \end{equation*}
    To analyze the convergence, we introduce an intermediate term $\mathbb{E}_{\mathbf{u}\sim P(\cdot \mid \mathbf{x})}[\hat{\psi}(\mathbf{x}, \mathbf{u}, \mathbf{a})]$, representing the expected value of the estimated inner function under the true outer distribution. By the triangle inequality, we decompose the total error into:
    \begin{equation*}
        \begin{aligned}
            \Delta &= \left| \hat{\mathbb{E}}_{\mathbf{u}\sim P(\cdot \mid \mathbf{x})}[\hat{\psi}] - \mathbb{E}_{\mathbf{u}\sim P(\cdot \mid \mathbf{x})}[\hat{\psi}] + \mathbb{E}_{\mathbf{u}\sim P(\cdot \mid \mathbf{x})}[\hat{\psi}] - \mathbb{E}_{\mathbf{u}\sim P(\cdot \mid \mathbf{x})}[\psi] \right| \\
            &\leq \underbrace{\left| \hat{\mathbb{E}}_{\mathbf{u}\sim P(\cdot \mid \mathbf{x})}[\hat{\psi}] - \mathbb{E}_{\mathbf{u}\sim P(\cdot \mid \mathbf{x})}[\hat{\psi}] \right|}_{\text{Term I: Outer Estimation Error}} + \underbrace{\left| \mathbb{E}_{\mathbf{u}\sim P(\cdot \mid \mathbf{x})}[\hat{\psi} - \psi] \right|}_{\text{Term II: Inner Estimation Error}}.
        \end{aligned}
    \end{equation*}

    \textbf{Analysis of Term I (Outer Estimation Error):}
    Term I captures the error arising from estimating the expectation of the fixed function $\hat{\psi}$ (conditioned on Stage 1 data) using the finite sample estimate of the outer distribution.
    Recall that our estimator is derived as $\hat{\mathbb{E}}_{\mathbf{u}\sim P(\cdot \mid \mathbf{x})}[\hat{\psi}] = \langle \hat{\psi}, \hat{\mu}_{\mathbf{U} \mid \mathbf{x}} \rangle_{\mathcal{H}_{k_u}}$, where $\hat{\mu}_{\mathbf{U} \mid \mathbf{x}}$ is the empirical CME of $\mathbf{u}$ given $\mathbf{x}$. Similarly, the intermediate term can be written as $\mathbb{E}_{\mathbf{u}\sim P(\cdot \mid \mathbf{x})}[\hat{\psi}] = \langle \hat{\psi}, \mu_{\mathbf{U} \mid \mathbf{x}} \rangle_{\mathcal{H}_{k_u}}$.
    Applying the Cauchy-Schwarz inequality:
    \begin{equation*}
        \begin{aligned}
            \text{Term I} &= \left| \langle \hat{\psi}, \hat{\mu}_{\mathbf{U} \mid \mathbf{x}} - \mu_{\mathbf{U} \mid \mathbf{x}} \rangle_{\mathcal{H}_{k_u}} \right| \\
            &\leq \| \hat{\psi} \|_{\mathcal{H}_{k_u}} \| \hat{\mu}_{\mathbf{U} \mid \mathbf{x}} - \mu_{\mathbf{U} \mid \mathbf{x}} \|_{\mathcal{H}_{k_u}}.
        \end{aligned}
    \end{equation*}
    The estimator $\hat{\psi}$ is obtained via KRR in the joint RKHS $\mathcal{H}_{\mathcal{H}}$. 
Under the decomposable kernel assumption, for any fixed $\mathbf{x}$ and $\mathbf{a}$, the section $\hat{\psi}(\mathbf{x}, \cdot, \mathbf{a})$ explicitly resides in $\mathcal{H}_{k_u}$. 
Moreover, its norm $\| \hat{\psi} \|_{\mathcal{H}_{k_u}}$ is bounded, which is a direct consequence of the boundedness of the regularized estimator in the joint space $\mathcal{H}_{\mathcal{H}}$.
    For the CME estimation error $\| \hat{\mu}_{\mathbf{U} \mid \mathbf{x}} - \mu_{\mathbf{U} \mid \mathbf{x}} \|_{\mathcal{H}_{k_u}}$, we invoke Lemma~\ref{lemma:cme_pointwise}. With the outer regularization parameter $\lambda_x$, we have:
    \begin{equation*}
        \text{Term I} = \mathcal{O}_p\left( \frac{1}{\sqrt{N \lambda_x}} + \sqrt{\lambda_x} \right).
    \end{equation*}

    \textbf{Analysis of Term II (Inner Estimation Error):}
    Term II reflects the propagated error from the inner expectation estimation. 
    % Since the outer expectation integrates over the entire domain of $\mathbf{u}$, we require a uniform bound on the estimation error of $\psi(\cdot)$.
    \begin{equation*}
        \begin{aligned}
            \text{Term II} &= \left| \int_\sU (\hat{\psi}(\mathbf{h}) - \psi(\mathbf{h})) p(\mathbf{u} \mid \mathbf{x}) d\mathbf{u} \right| \\
            &\leq \sup_{\mathbf{u} \in \mathcal{U}} \left| \hat{\psi}(\mathbf{x}, \mathbf{u}, \mathbf{a}) - \psi(\mathbf{x}, \mathbf{u}, \mathbf{a}) \right|.
        \end{aligned}
    \end{equation*}
    According to reproducing property, $\psi(\mathbf{h}) = \langle w_\eta, \mu_{\mathbf{Y} \mid \mathbf{h}} \rangle_{\mathcal{H}_{k_y}}$ and $\hat{\psi}(\mathbf{h}) = \langle w_\eta, \hat{\mu}_{\mathbf{Y} \mid \mathbf{h}} \rangle_{\mathcal{H}_{k_y}}$. By Cauchy-Schwarz inequality:
    \begin{equation*}
        \left| \hat{\psi}(\mathbf{h}) - \psi(\mathbf{h}) \right| \leq \| w_\eta \|_{\mathcal{H}_{k_y}} \| \hat{\mu}_{\mathbf{Y} \mid \mathbf{h}} - \mu_{\mathbf{Y} \mid \mathbf{h}} \|_{\mathcal{H}_{k_y}}.
    \end{equation*}
    To bound this uniformly for all $\mathbf{h}$, we rely on the uniform consistency result from Lemma~\ref{lemma:cme_uniform}. With the inner regularization parameter $\lambda_h$, we obtain:
    \begin{equation*}
        \sup_{\mathbf{h}} \| \hat{\mu}_{\mathbf{Y} \mid \mathbf{h}} - \mu_{\mathbf{Y} \mid \mathbf{h}} \|_{\mathcal{H}_{k_y}} = \mathcal{O}_p\left( \sqrt{\lambda_h} + \frac{1}{\sqrt{N \lambda_h^3}} \right).
    \end{equation*}
    Thus, Term II is bounded by this uniform rate.
    Combining the bounds for Term I and Term II, we obtain the total bound:
    \begin{equation*}
        \Delta \leq \text{Term I} + \text{Term II} = \mathcal{O}_p\left( \sqrt{\lambda_x} + \sqrt{\lambda_h} + \frac{1}{\sqrt{N \lambda_x}} + \frac{1}{\sqrt{N \lambda_h^3}} \right).
    \end{equation*}
\end{proof}

%% file: appendix/experimental_details.tex
\section{Experimental details and additional results}
\label{app:experiment}
In this section, we provide details for the experiments in Sec.~\ref{sec:experiment}. 
Hyperparameters like bandwidths are selected empirically based on heuristics in the kernel literature~\citep{gretton12a}, and are available in the provided supplementary code.

\vspace{-0.1in}
\subsection{Linear data}
\paragraph{Bermuda data.} The Bermuda data is an environmental dataset that records a collection of marine and biogeochemical variables measured in the Bermuda region~\citep{courtney2017environmental}. 
The variables involved in the generation process include:
\vspace{-0.1in}
\begin{itemize}[leftmargin=24pt, itemsep=0.03in]
  \item Light: Light levels at the ocean bottom; 
  \item Temp: Temperature at the ocean bottom; 
  \item Sal: Sea surface salinity; 
  \item DIC: Dissolved inorganic carbon in seawater; 
  \item TA: Total alkalinity of seawater; 
  \item $\Omega_A$: Aragonite saturation state of seawater; 
  \item Chla: Chlorophyll-\textit{a} concentration at the sea surface; 
  \item Nut: The first principal component (PC1) of NH$_4$, NO$_2$ + NO$_3$, and SiO$_4$;
  \item pHsw: Seawater pH; 
  \item CO$_2$: Partial pressure of CO$_2$ in seawater ($P_{\text{CO}_2}$); 
  \item NEC: Net ecosystem calcification.
\end{itemize}
\vspace{-0.1in}
The rehearsal graph over these variables is illustrated in \Fig~\ref{fig:bermuda}. 
Following~\citet{aglietti2020causal,qint2023srm}, the actionable variables that can be altered by the decision-maker are DIC, TA, $\Omega_A$, Chla, and Nut, which can be altered into values with constraint \([-1.0, 1.0]\). The desired region of is $\sS=\{\text{NEC}\in[0.5, 2]\}$, following the specifications in \Sec~\ref{sec:experiment}.

\begin{figure}[h]
    \centering
    \includegraphics[width=0.8\linewidth]{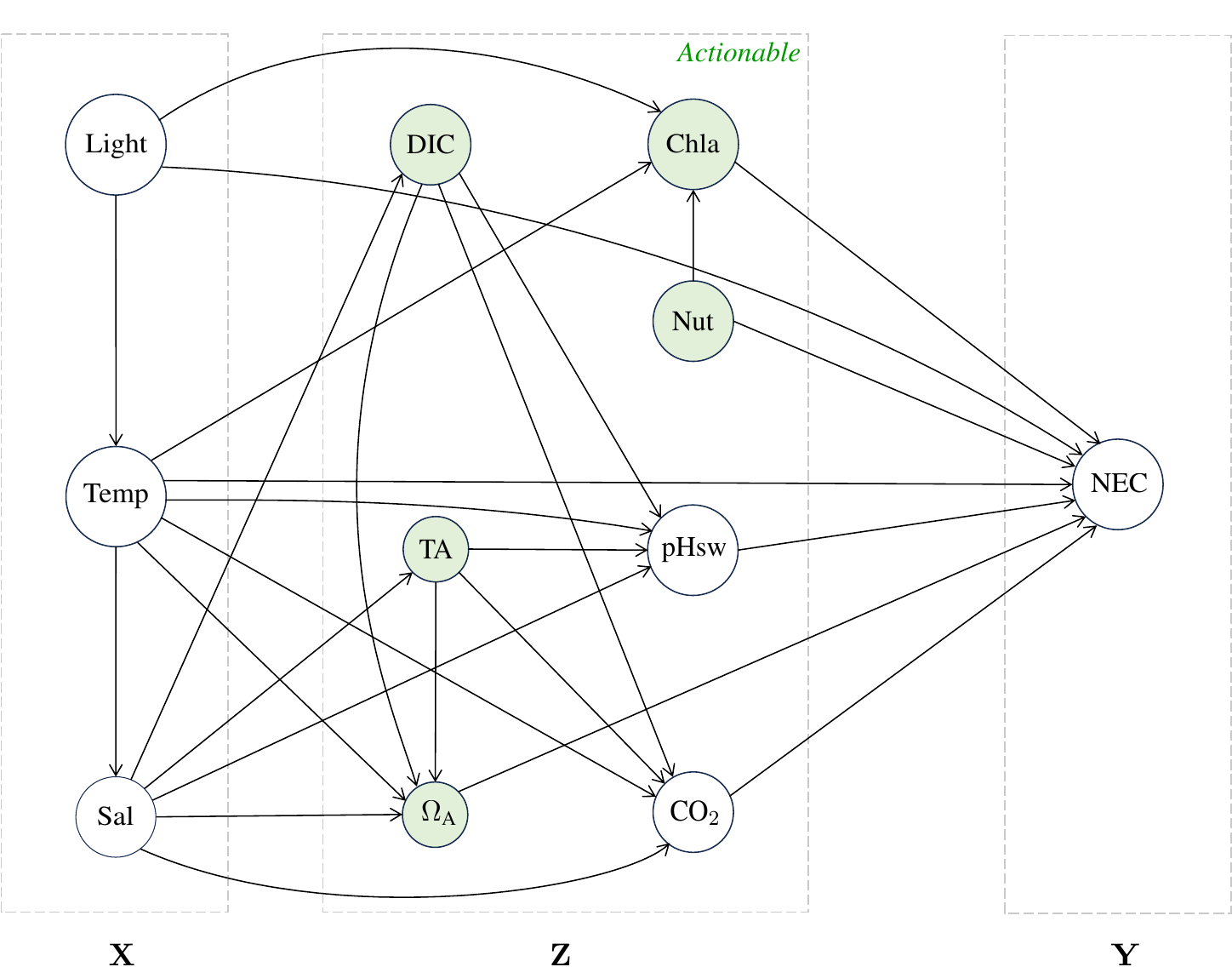}
    \caption{The graph structure of Bermuda data.}
    \label{fig:bermuda}
\end{figure}

\paragraph{Linear synthetic data.} 
The rehearsal graph over the variables in our linear synthetic dataset is illustrated in \Fig~\ref{fig:lin-syn}.
The desired region of \(\Y\) is specified as
\(\{(Y_1, Y_2) \mid 0.0 \leq Y_1 \leq 2.0,\; 0.0 \leq Y_2 \leq 2.0\}\), and the actionable variables are $A_1$ and $A_2$ with \(A_1, A_2\in[-3.0, 3.0]\).
The generation mechanisms underlying the natural observational process are given by:
% \vspace{-0.05in}

\[
\hspace{-3in}\left\{
\begin{aligned}
    &\quad  X_1 := \sN(0, 0.1); \\
    &\quad  X_2 := \sN(0, 0.1); \\
    &\quad  U_2 := 10.0 \times X_2 + \sN(0, 0.1); \\
    &\quad  A_1 := 10.0 \times X_1 + \sN(0, 0.1); \\
    &\quad  U_1 := 0.5 \times A_1 + 1.3 \times U_2 + \sN(0, 0.1); \\
    &\quad  A_2 := 2.0 \times A_1 + 0.4 \times U_2 + \sN(0, 0.1); \\
    &\quad  Y_1 := -1.0 \times A_1 + 0.9 \times A_2 + \sN(0, 0.1); \\
    &\quad  Y_2 := 1.6 \times A_1 - 0.9 \times A_2 + \sN(0, 0.1).
\end{aligned}
\right.
\]

\begin{figure}[h]
    \centering
    % \vspace{-0.1in}
    \includegraphics[width=0.8\linewidth]{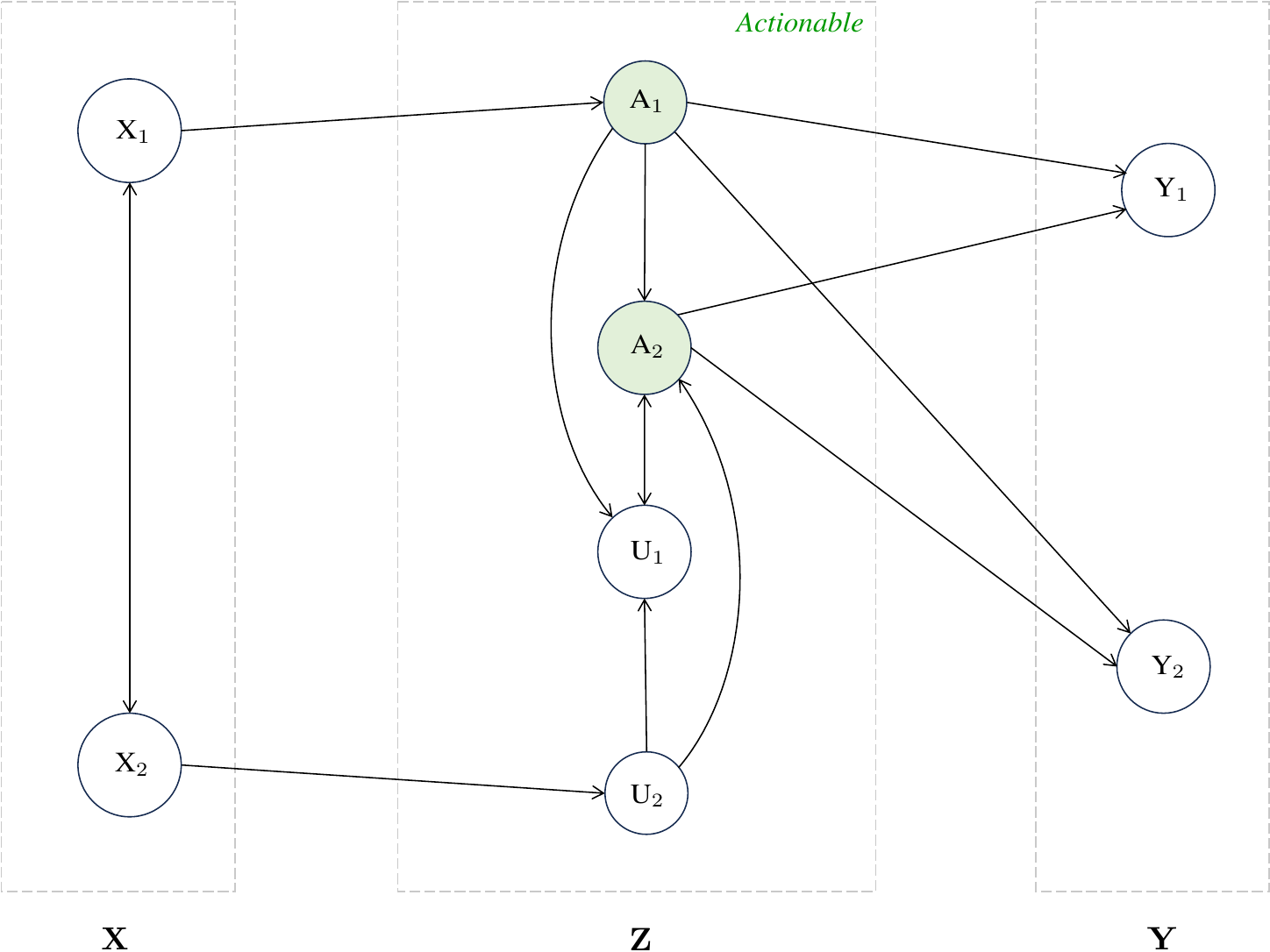}
    \caption{The graph structure of linear synthetic data.}
    \label{fig:lin-syn}
\end{figure}
\vspace{-0.2in}

\subsection{Nonlinear data}
\paragraph{Bank data.} 
The rehearsal graph over the variables in our bank dataset is illustrated in \Fig~\ref{fig:auf_example}. 
The actionable variables are formally defined as $A_1$ and $A_2$ (note that $A_1$ is excluded in our implementation to facilitate the visualization in \Fig~\ref{fig:auf_bank}).
The desired region of \(\Y\) is $\sS = \{ (Y_1, Y_2) \mid Y_1 \geq 0.6, Y_2 \geq 0.3 \}$, following the specifications in \Sec~\ref{sec:experiment}.
The variables involved in the decision process, along with the generation mechanisms underlying the natural observational process, are given by:
\vspace{-0.1in}
\begin{itemize}[leftmargin=24pt, itemsep=0.03in]
    \item $X_1$: Credit Score ($X_1 := \mathcal{U}(0, 1)$);
    \item $X_2$: Debt-to-Income Ratio ($X_2 := \mathcal{U}(0, 1)$);
    \item $U_1$: Borrower's Industry Sector Stability Index ($U_1 \sim \operatorname{Beta}(2, 2)$);
    \item $A_1$: Loan Amount Granted (Excluded in implementation);
    \item $A_2$: Interest Rate Applied ($A_2 := U_1 + 0.5 X_1 + 0.5 X_2 - 0.5 + \mathcal{N}(0, 0.2)$) with \(A_2\in[0.0, 1.0]\);
    \item $Y_1$: Repayment Rate ($Y_1 := \operatorname{Sigmoid}(2.0\times U_1^{1.1} - 1.5\times A_2 + 0.2\times X_2 + 0.4) + \mathcal{N}(0, 0.05)$);
    \item $Y_2$: Bank ROI (Return on Investment) ($Y_2 := 0.8 A_2 + 0.5 U_1 + \mathcal{N}(0, 0.05)$).
\end{itemize}

\paragraph{Nonlinear synthetic data 1.} The rehearsal graph over the variables in our nonlinear synthetic dataset 1 can be described as: \(X\rightarrow A_1, X\rightarrow Y\); \(U\rightarrow A_1, U\rightarrow A_2, U\rightarrow Y\); \(A_1\rightarrow Y\); \(A_2\rightarrow Y\). Desired region of \(\Y\) is specified as
\(\{Y \mid 1.5 \leq Y \leq 2.0\}\), and actionable variables are $A_1$ and $A_2$ with \(A_1, A_2\in[-1.0, 1.0]\). Generation mechanisms underlying observational process are given by:

\[
\hspace{-1in}\left\{
\begin{aligned}
    &\quad  X := \mathcal{U}(-1, 1); \\
    &\quad  U := \operatorname{Exp}(1); \\
    &\quad  A_1 := 0.8 X + 0.2 U + \sN(0, 0.5); \\
    &\quad  A_2 := 0.5 \sin(U) + \sN(0, 0.5); \\
    &\quad  Y := 1.5 - (A_1 - X)^2 - (A_2 - \ln(U + 1))^2 + 0.2 \sin(A_1 A_2) + \sN(0, 0.1).
\end{aligned}
\right.
\]
\paragraph{Nonlinear synthetic data 2.} 
The rehearsal graph over the variables in our nonlinear synthetic dataset 2 is illustrated in \Fig~\ref{fig:nonlin-syn2}.
The desired region of \(\Y\) is specified as
\(\{Y_1 \mid 0.9 \leq Y_1 \leq 1.5\}\), and the actionable variables are $A_1$ and $A_2$ with \(A_1, A_2\in[-1.0, 1.0]\). 
The generation mechanisms underlying the natural observational process are given by:
\[
\left\{
\begin{aligned}
    &\quad  X_1 := \mathcal{U}(0, 1); \\
    &\quad  X_2 := \mathcal{U}(0, 1); \\
    &\quad  A_1 := 1.0 X_1 - 1.0 X_2 + 0.2 X_1^2 - 0.2 X_2^2 + 0.1 X_1 X_2 + \sN(0, 0.1); \\
    &\quad  U_1 := -1.0 X_1 + 2.0 X_2 + 3.0 A_1 - 0.2 X_1^2 + 0.4 X_2^2 + 0.6 A_1^2 + 0.1 (X_1 X_2 + X_1 A_1 \\
    &\quad\quad\quad + X_2 A_1) + \sN(0, 0.1); \\
    &\quad  U_2 := -1.0 X_1 + 4.0 U_1 - 0.2 X_1^2 + 0.8 U_1^2 + 0.1 X_1 U_1 + \sN(0, 0.1); \\
    &\quad  A_2 := -1.0 X_1 - 0.5 X_2 + 0.3 U_1 - 0.2 X_1^2 - 0.1 X_2^2 + 0.06 U_1^2 + 0.1 (X_1 X_2 + X_1 U_1 \\
    &\quad\quad\quad + X_2 U_1) + \sN(0, 0.1); \\
    &\quad  Y_1 := 0.5 + \frac{1}{60} \Big( 0.2 (X_1 + X_1^2) - 5.0 (U_1 + U_1^2) - 1.0 (U_2 + U_2^2) + 5.0 (A_2 + A_2^2) \Big) \\
    &\quad\quad\quad + 0.5 \sum_{p_i, p_j \in \mathbf{Pa}(Y_1), i<j} p_i p_j + \sN(0, 0.1).
\end{aligned}
\right.
\]

\begin{figure}[h]
    \centering
    % \vspace{-0.1in}
    \includegraphics[width=0.82\linewidth]{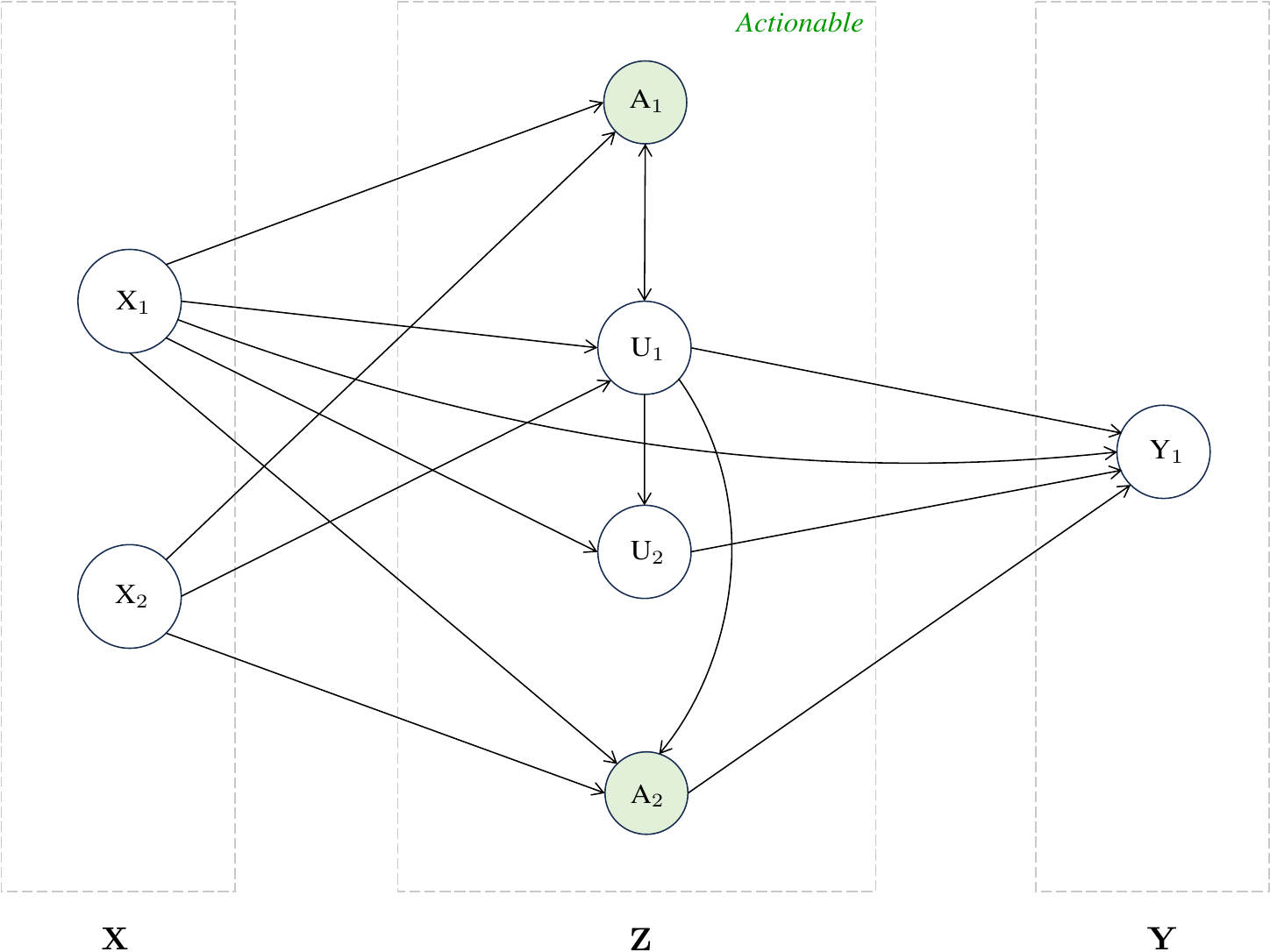}
    \caption{The graph structure of nonlinear synthetic data 2.}
    \label{fig:nonlin-syn2}
\end{figure}

% \begin{itemize}
%     \item original version: \(\bbP(\Y\mid \X=\x, Rh(\Z=\z)), \bbP(\Y\mid \x, Rh(\z))\) original version
%     \item original version: \[\bbP(\Y\mid \X=\x, Rh(\Z=\z)), \bbP(\Y\mid \x, Rh(\z))
%     \]
%     \item candidate 1: \(\bbP(\Y\mid \X=\x, \Z\aleq\z), \bbP(\Y\mid \x, \mathring{\z})\) candidate 1
%     \item candidate 1: \[\bbP(\Y\mid \X=\x, \Z\aleq\z), \bbP(\Y\mid \x, \mathring{\z})\]
%     \item candidate 2: \(\bbP(\Y\mid \X=\x, \Z\rheq\z), \bbP(\Y\mid \x, \mathring{\z})\) candidate 2
%     \item candidate 2: \[\bbP(\Y\mid \X=\x, \Z\rheq\z), \bbP(\Y\mid \x, \mathring{\z})\]
%     \item candidate 3: \(\bbP(\Y\mid \X=\x, \Z\rheq\z), \bbP(\Y\mid \x, \mathop{\z}\limits^{r})\) candidate 3
%     \item candidate 3: \[\bbP(\Y\mid \X=\x, \Z\rheq\z), \bbP(\Y\mid \x, \mathop{\z}\limits^{r})\]
% \end{itemize}

%% file: appendix/scalability_sensitivity.tex
\subsection{Scalability and sensitivity experiments}
\label{app:scalability_sensitivity}

We report additional simple experiments on Non-Syn1 to evaluate computational scalability and hyperparameter sensitivity.
All AUF probabilities in this section are averaged over $5$ random seeds.

\paragraph{Kernel approximation for scalability.}
The exact nested KRR estimator requires solving Gram-matrix systems in the offline fitting stage.
This cost is paid before deployment, while online decision-making only requires evaluating the fitted surrogate and running projected gradient ascent.
To test whether standard kernel approximations can reduce the offline cost, we compare exact KRR with Nystr\"om~\citep{williams2001nystrom} and random Fourier feature (RFF)~\citep{rahimi2007random} approximations in Tab.~\ref{tab:scalability_approx}.

\begin{table}[h]
    \centering
    \caption{Scalability results on Non-Syn1. Time is the offline fitting time in seconds.}
    \label{tab:scalability_approx}
    \begin{tabular}{lcccc}
        \toprule
        Method & $N$ & Rank/Dim. & Time (s) & AUF Prob. ($\uparrow$) \\
        \midrule
        Exact & 1k & -- & 0.131 & 0.424 \\
        Nystr\"om & 1k & 100 & 0.022 & 0.424 \\
        Nystr\"om & 1k & 200 & 0.035 & 0.424 \\
        RFF & 1k & 500 & 0.087 & 0.360 \\
        Exact & 5k & -- & 2.118 & 0.453 \\
        Nystr\"om & 5k & 100 & 0.292 & 0.453 \\
        \bottomrule
    \end{tabular}
\end{table}

Nystr\"om preserves the AUF probability of exact KRR in these runs while reducing offline time by roughly $4\times$--$7\times$.
RFF also reduces runtime but shows a clearer speed-performance trade-off in this setting.
These results indicate that the proposed estimator can directly benefit from standard low-rank kernel approximations when larger observational datasets are available.

\paragraph{Sensitivity to the smoothing parameter.}
The smoothing parameter $\eta$ controls the approximation of the desired-region indicator by $w_\eta$.
Table~\ref{tab:eta_sensitivity} shows that moderate smoothing works best on Non-Syn1.
Very small $\eta$ over-smooths the indicator, while very large $\eta$ makes the surrogate sharper and harder to estimate from finite samples, matching the approximation-estimation trade-off discussed in Sec.~\ref{sec:theory}.

\begin{table}[h]
    \centering
    \caption{Sensitivity to the smoothing parameter $\eta$ on Non-Syn1.}
    \label{tab:eta_sensitivity}
    \begin{tabular}{lccccc}
        \toprule
        $\eta$ & 3 & 5 & 10 & 20 & 50 \\
        \midrule
        Mean AUF Prob. ($\uparrow$) & 0.366 & 0.370 & 0.456 & 0.436 & 0.366 \\
        \bottomrule
    \end{tabular}
\end{table}

\paragraph{Sensitivity to regularization.}
We further fix $\eta=10$ and sweep the inner and outer KRR regularization parameters $\lambda_h$ and $\lambda_x$ over multiplicative factors of their default values.
As shown in Tab.~\ref{tab:regularization_sensitivity}, the method exhibits a broad plateau around the default scale: many combinations between $0.2\times$ and $2\times$ achieve nearly identical AUF probabilities.
Performance mainly degrades when both regularization levels are made substantially larger.

\begin{table}[h]
    \centering
    \caption{Sensitivity to regularization on Non-Syn1. Entries are mean AUF probabilities with $\eta=10$.}
    \label{tab:regularization_sensitivity}
    \begin{tabular}{lccccc}
        \toprule
        $\lambda_h/\lambda_x$ & $0.2\times$ & $0.5\times$ & $1\times$ & $2\times$ & $5\times$ \\
        \midrule
        $0.2\times$ & 0.456 & 0.456 & 0.456 & 0.456 & 0.448 \\
        $0.5\times$ & 0.456 & 0.456 & 0.456 & 0.448 & 0.448 \\
        $1\times$   & 0.456 & 0.456 & 0.456 & 0.448 & 0.376 \\
        $2\times$   & 0.456 & 0.448 & 0.448 & 0.376 & 0.352 \\
        $5\times$   & 0.448 & 0.448 & 0.382 & 0.352 & 0.358 \\
        \bottomrule
    \end{tabular}
\end{table}

%% file: appendix/nhanes_benchmark.tex
\subsection{NHANES benchmark}
\label{app:nhanes_benchmark}

We build one real-data-derived semi-synthetic AUF benchmark from the U.S. National Health and Nutrition Examination Survey (NHANES) 2011--2018 cycles~\citep{nhanes_cdc}. 
The benchmark is designed as a contextual diabetes-control task: after observing immutable demographic and socioeconomic context variables, a method chooses feasible values for actionable lifestyle/metabolic variables, and the target outcomes are two glycemic markers, glycated hemoglobin (HbA1c) and fasting plasma glucose (FPG). 
Because NHANES is cross-sectional, post-alteration outcomes are generated by a fitted structural generator rather than by randomized clinical interventions; the benchmark should therefore be interpreted as a semi-synthetic evaluation environment.

\paragraph{Data source and cohort.}
The raw NHANES modules are merged by participant identifier (\texttt{SEQN}) across the 2011--2012, 2013--2014, 2015--2016, and 2017--2018 cycles. 
We harmonize cross-cycle aliases such as \texttt{BPXOSY1}/\texttt{BPXSY1} and \texttt{BPXODI1}/\texttt{BPXDI1}, remove NHANES sentinel missing codes, keep adults with both outcomes observed, and apply complete-case selection on the benchmark variables. 
The resulting benchmark has $7{,}465$ complete-case participants.

\begin{table}[h]
    \centering
    \scriptsize
    \caption{NHANES benchmark summary. The validation column reports the mean absolute error between real and generated correlation matrices.}
    \label{tab:nhanes_summary}
    \begin{tabular}{lcllcc}
        \toprule
        Benchmark & Complete cases & Actionable variables $\Z_a$ & Non-actionable intermediates $\Z_u$ & Outcomes \(\Y\) & Corr. MAE \\
        \midrule
        NHANES & 7,465 & \makecell[l]{BMI, CalIntake, CarbIntake,\\ FiberIntake, Sedentary\_min} & \makecell[l]{SBP, DBP,\\ TotalCholesterol, HDL} & HbA1c, FPG & 0.0483 \\
        \bottomrule
    \end{tabular}
\end{table}

\paragraph{Actionability and influence relations.}
We partition variables according to the temporal and physiological ordering available in the cross-sectional measurements: immutable context variables $\X$ precede intermediate variables $\Z$, and glycemic outcomes $\Y$ are downstream targets. 
The immutable context variables are Age, Sex, Race, FamilyHx\_Diabetes, Education, and Income\_Ratio. 
Expert actionability assumptions mark BMI, daily calorie intake, carbohydrate intake, fiber intake, and sedentary minutes as actionable. 
Blood pressure and lipid measurements are treated as non-actionable because the benchmark models them as physiological consequences rather than direct decision levers.
The benchmark DAG encodes the following influence relations. Context variables influence BMI, diet, sedentary behavior, blood pressure, lipids, and both outcomes. BMI influences diet, blood pressure, lipids, HbA1c, and FPG. Calorie intake influences carbohydrate/fiber intake, lipids, HbA1c, and FPG. Fiber intake and sedentary time influence both glycemic outcomes, and the physiological intermediates SBP, total cholesterol, and HDL influence both outcomes.

\begin{figure}[h]
    \centering
    \includegraphics[width=0.8\linewidth]{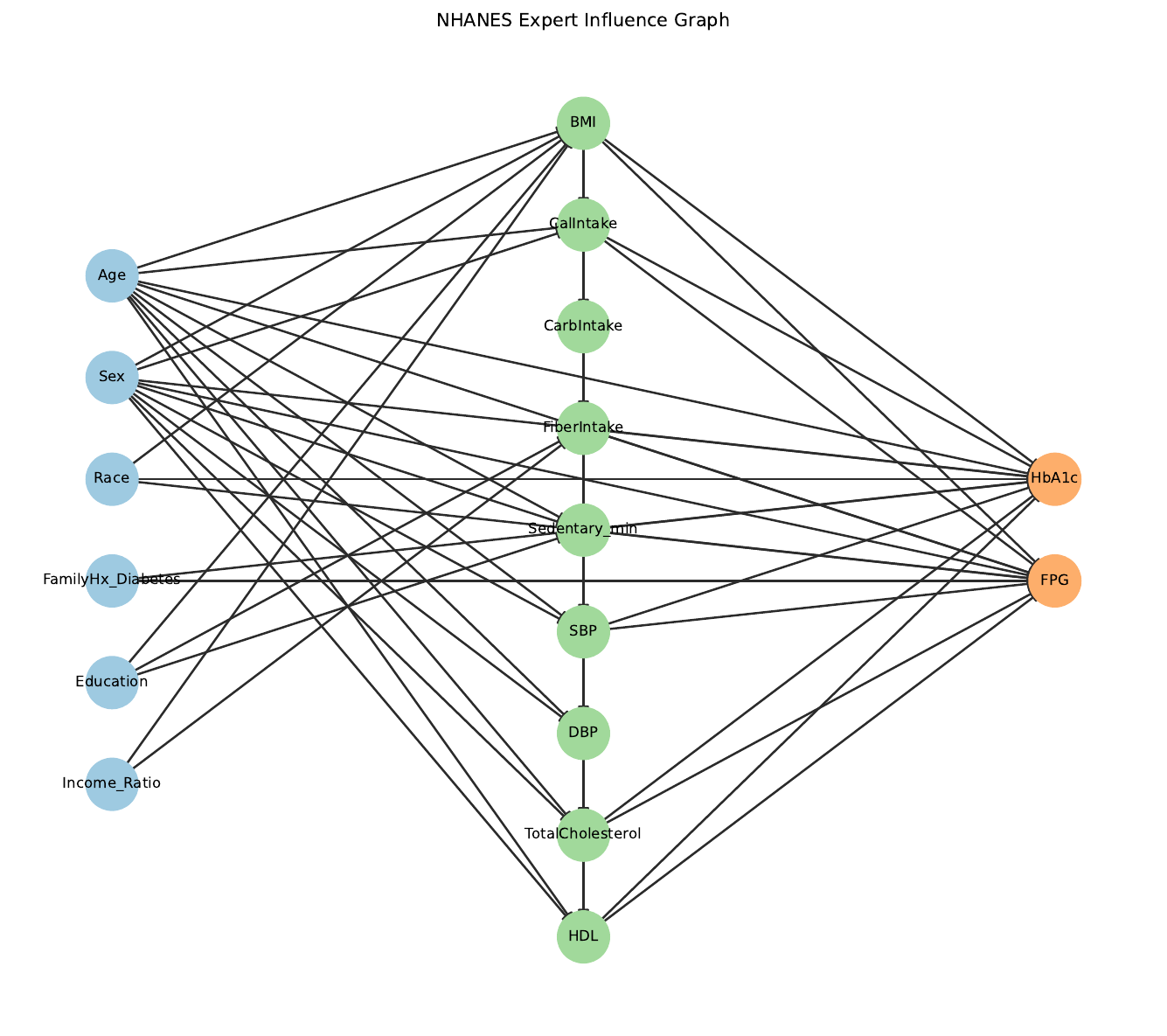}
    \vspace{-0.1in}
    \caption{NHANES influence graph specified by human experts for the benchmark. Blue nodes denote context variables; green nodes denote intermediate variables, where BMI, CalIntake, CarbIntake, FiberIntake, and Sedentary\_min are actionable and SBP, DBP, TotalCholesterol, and HDL are non-actionable; orange nodes denote glycemia outcomes.}
    \label{fig:nhanes_structure}
\end{figure}
\vspace{-0.1in}

\paragraph{Typed structural generator.}
For each node, the generator fits a conditional model suited to the variable type. 
The parent sets for these node-wise models follow the expert-specified influence graph in Fig.~\ref{fig:nhanes_structure}.
Discrete variables use empirical categorical sampling for root nodes and gradient-boosted classifiers for non-root nodes. 
Continuous variables use gradient-boosted regressors with residual bootstrap; nonnegative variables are modeled on a log-transformed scale and clipped back to valid support. 
During evaluation, an alteration sets only a subset of $\Z_a$ and is clipped to the empirical $[1\%,99\%]$ quantile interval of the corresponding variable.

\paragraph{Success region.}
We use one clinically interpretable two-dimensional glycemic target:
\[
    \sS_{\text{NHANES}}=\{(\text{HbA1c}, \text{FPG}) \mid \text{HbA1c}<5.7,\ \text{FPG}<100~\text{mg/dL}\}.
\]
These thresholds correspond to common normoglycemia screening boundaries~\citep{ada2024diagnosis}. 
Success is estimated as the Monte Carlo frequency with which simulated post-alteration outcomes fall inside this region.

% \paragraph{NHANES result table.}
% Table~\ref{tab:nhanes_method_results} follows the AUF-probability reporting protocol used in Tab.~\ref{tab:main_results}. 
% Each method is fitted from $1{,}000$ observational samples and evaluated over $5$ random seeds with $100$ Monte Carlo samples per seed.
% For our method, the NHANES adapter uses tuned kernel regularization and bandwidth multipliers recorded in the experiment artifact while keeping the final evaluation contexts and Monte Carlo seeds shared across all methods.
% Because several comparison methods assume linear-Gaussian internal models, their NHANES entries use lightweight adapters fitted from the same observational samples; final table values are always computed by the fitted NHANES generator rather than by a method's internal estimated objective.

\paragraph{Validation against observed statistics.}
The fitted generator is validated by comparing generated samples against the processed complete-case data. 
The current artifacts report correlation-matrix MAE of $.0483$. 
The marginal-distribution overlays and correlation comparisons in Fig.~\ref{fig:nhanes_validation} show that the generated observational samples broadly preserve the empirical distributional shape and pairwise dependence structure, while retaining expected limitations of a fitted observational generator in tails.

\begin{figure}[h]
    \centering
    \includegraphics[width=\linewidth]{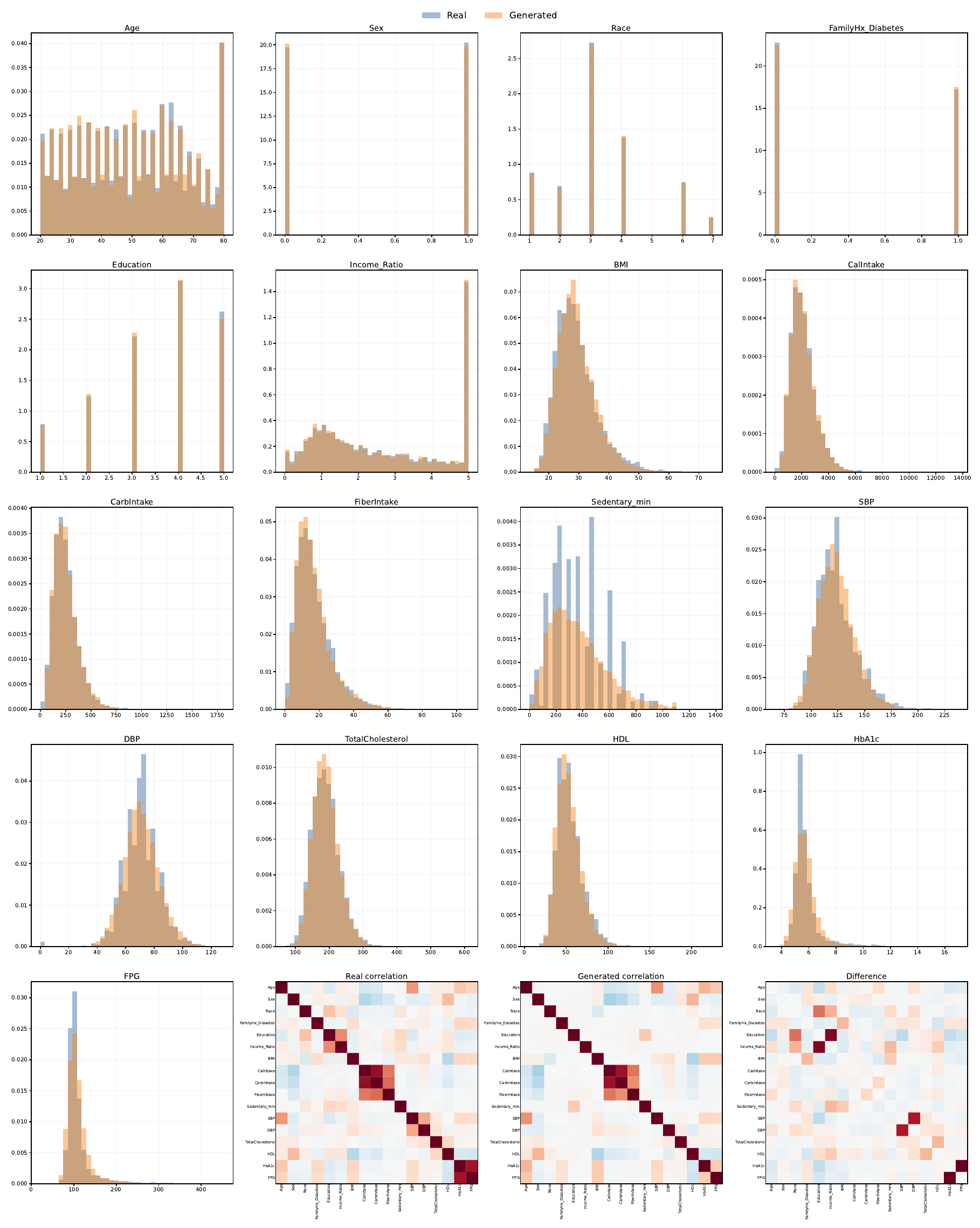}
    \caption{NHANES generator validation plots comparing real complete-case samples against generated observational samples. In the histogram panels, blue histograms denote real samples and orange histograms denote generated samples. The three heatmaps show the real correlation matrix, the generated correlation matrix, and their difference, respectively.}
    \label{fig:nhanes_validation}
\end{figure}

\paragraph{Limitations.}
The benchmark does not model NHANES survey weights, strata, or primary sampling units. 
Complete-case filtering may induce selection bias. 
Most importantly, the generator is fitted from observational cross-sectional data, so its post-action distributions are benchmark ground truth under the specified SEM rather than clinical influence truth.